\documentclass[twoside,11pt]{article}

\usepackage{jmlr2e}
\usepackage{shortcuts}
\usepackage{algorithm2e}

\usepackage{bold-extra}
\usepackage{rotating}

\usepackage{bold-extra}
\usepackage{courier}
 \usepackage{booktabs}

 \usepackage{slantsc}
 \usepackage{lmodern}

\usepackage{scrextend}

\usepackage{multirow}

\usepackage{listings}
\usepackage{xcolor}

\newbox\flinebox
\newbox\slinebox
\newbox\mlinebox
\def\duplines{\setlength\parindent{0pt}
  \setbox\flinebox\lastbox
  \ifvoid\flinebox\relax
  \else
  \setbox\slinebox\hbox{\copy\flinebox}
  \setbox\mlinebox\hbox{\copy\flinebox}
  \unskip\unpenalty
  {\duplines}
  \color{black!95}\box\flinebox\vspace*{-2.8ex}
  \box\slinebox \fi
}

\jmlrheading{17}{2017}{1-42}{04/17}{--/--}{C\'{e}dric Malherbe and Nicolas Vayatis}

\ShortHeadings{Global optimization of Lipschitz functions}{Malherbe and Vayatis}

\newtheorem{condition}{Condition}

\newcommand\blankfootnote[1]{
  \let\thefootnote\relax\footnotetext{#1}
  \let\thefootnote\svthefootnote
}

\begin{document}


\title{Global optimization of Lipschitz functions \vspace{1em}}

\author{\name C\'{e}dric Malherbe \email malherbe@cmla.ens-cachan.fr
  \AND \name Nicolas Vayatis \email vayatis@cmla.ens-cachan.fr
  \AND \addr CMLA - Ecole Normale Sup\'{e}rieure de Cachan\\
  CNRS - Universit\'{e} Paris-Saclay\\
  94 235 Cachan cedex, France}

\editor{-}

\maketitle

\vspace{-0.7em}
\begin{abstract}
The goal of the paper is to design sequential strategies which lead 
to efficient optimization of an unknown function under the only 
assumption that it has a finite Lipschitz constant. 
We first identify sufficient conditions for the consistency 
of generic sequential algorithms and formulate the expected minimax rate for their performance.
We introduce and analyze a first algorithm called LIPO which assumes the Lipschitz 
constant to be known. 
Consistency, minimax rates for LIPO are proved, 
as well as fast rates under an additional H\"older like condition. 
An adaptive version of LIPO is also introduced for the more realistic 
setup where the Lipschitz constant is unknown and has to be estimated along with the optimization. 
Similar theoretical guarantees are shown to hold for the adaptive LIPO algorithm 
and a numerical assessment is provided at the end of the paper
to illustrate the potential of this strategy 
with respect to state-of-the-art methods over typical benchmark problems for global optimization.
\end{abstract} 


\section{Introduction}
\label{submission}
In many applications such as complex system design or hyperparameter
calibration for learning systems, the goal is to optimize the output 
value of an unknown function with as few evaluations as possible.
Indeed, in such contexts, evaluating the performance
of a single set of parameters often requires 
numerical simulations or cross-validations with significant computational cost
and the operational constraints impose a sequential exploration of
the solution space with small samples. Moreover,
it can ge\-ne\-rally not be assumed that the function
has good properties such as linearity or convexity.
This generic problem of sequentially optimizing the output of an unknown and potentially
nonconvex function is often referred to as global optimization (\cite{pinter1991global}),
black-box optimization (\cite{jones1998efficient})
or derivative-free optimization (\cite{rios2013derivative}).
In particular, there is a large number of algorithms based on various
heuristics which have been introduced in order to address this problem,
such as genetic algorithms, 
model-based methods or Bayesian optimization. 
We focus here on
the smoothness-based approach to global optimization.
This approach is based on the simple observation that, in many applications,
the system presents some regularity with respects to the input.
In particular, the use of the Lipschitz constant,
first proposed in the seminal works of \cite{shubert1972sequential} 
and \cite{piyavskii1972algorithm}, initiated an active line of research and 
played a major role in the development of many efficient
global optimization algorithms such as DIRECT (\cite{jones1993lipschitzian}),
MCS (\cite{huyer1999global}) or more recently SOO (\cite{preux2014bandits}).
Convergence properties of global optimization methods have been developed 
in the works of \cite{valko2013stochastic} and \cite{munosmono} 
under local smoothness assumptions, but, up to our knowledge, 
such properties have not been considered in the case where only the global smoothness
of the function can be specified. An interesting question is how much 
global assumptions on regularity which cover in some sense local assumptions 
may improve the convergence of the latter.
In this work, we address the following questions:
(i) find the limitations and the best performance that can be achieved 
by any algorithm over the class of Lipschitz functions  and (ii) design efficient
and optimal algorithms for this class of problems.
More specifically, our contribution with regards to the above mentioned works
is twofold. First, we  introduce two novel algorithms for global optimization 
which exploit the global smoothness of the unknown function and
display good performance in typical benchmarks for optimization.
Second,
we show that these algorithms 
can achieve faster rates of convergence on globally smooth problems
than the previously known methods which only exploit the local smoothness of the function.
The rest of the paper is organized as follows.
In Section \ref{sec:global_opt}, we introduce the framework and provide generic results
about the convergence of global optimization algorithms.
In Section \ref{sec:lipopt}, we introduce and analyze the LIPO algorithm
which requires the knowledge of the Lipschitz constant.
In Section \ref{sec:adalipopt}, the algorithm is extended to the case
where the Lipschitz constant is not assumed to be previously known.
Finally, the adaptive version of the algorithm is compared
to other global optimization methods in Section \ref{sec:experiments}.
All proofs are postponed to the Appendix Section.

\section{Setup and preliminary results}
  \label{sec:global_opt}
  
  \subsection{Setup and notations}

%
  
  {\bf Setup.}
  Let $\X \subset \R^d$ be a compact and convex set with non-empty interior
  and let $f:\X \rightarrow \R$ be an unknown function which is only supposed
  to admit a  maximum over its input domain.
  The goal in global optimization consists in finding some point
  \[
    x^{\star} \in \underset{x \in \X}{\arg \max}~ f(x)
  \]
  with a minimal amount of function evaluations.
  The standard setup involves a sequential procedure which 
  starts by evaluating the function $f(X_1)$ at an initial point $X_1$ 
  and iteratively selects at each step $t \geq 1$ an evaluation point $X_{t+1}\in\X$ 
  which depends on the previous evaluations
  $(X_1, f(X_1)), \dots, (X_t, f(X_t)) $ and receives the evaluation
  of the unknown function $f(X_{t+1})$ at this point.
  After $n$ iterations, we consider that the algorithm returns an evaluation point
  $
   X_{\hatin} 
  $
  with $\hatin \in \arg\min_{i=1\dots n}f(X_i)$
  which has recorded the highest evaluation.
  The performance of the algorithm over the function $f$ is then
  measured after $n$ iterations through the difference between the value of the true maximum 
  and the value of the highest evaluation observed so far:
  \[
   \max_{x \in \X} f(x) - \max_{i=1 \dots n}f(X_i).
  \]
  The analysis provided in the paper considers   that the number $n$ of
  evaluation points is not fixed and it is assumed that function
  evaluations are noiseless.
  Moreover, the assumption made on the unknown function $f$
  throughout the paper is that it has 
  a finite Lipschitz constant $k$, {\it i.e.}
  \[
   \exists k\geq0~~\normalfont{\text{s.t.~}} 
   |f(x) -f(x')| \leq k \cdot \norm{x-x'}_2~
   \forall (x,x')\in\X^2\!.
  \]
  Before starting the analysis, we point out that 
  similar settings have also been studied in the works of 
  \cite{munosmono} and \cite{malherbe}
  and that
  \cite{valko2013stochastic} and \cite{grill2015black} considered the noisy scenario.\\

  \noindent {\bf Notations.} For all $x=(x_1, \ldots, x_d) \in \R^d$, 
  we denote by $\norm{x}^2_2= \sum_{i=1}^d x_i^2$ the standard $\ell_2$-norm  and
  by $B(x,r)=\{x' \in \R^d: \norm{x-x'}_2 \leq r \}$
  the  ball centered in $x$ of radius $r\geq 0$.
  For any bounded set $\X \subset \R^d$, we define its inner-radius as
  $\textrm{rad}(\X)=\max \{r>0: \exists x\in \X \textrm{~such that~} B(x,r)\subseteq \X \}$,
  its diameter as $\diam{\X}=\max_{(x,x')\in \X^2}\norm{x-x'}_2$
  and we denote by $\mu(\X)$ its volume where $\mu(\cdot)$ stands for the Lebesgue measure.
  In addition,  $\text{Lip}(k) := \{f : \X \to \R \text{~s.t.~} | f(x) -f(x') |
  \leq k \cdot \norm{x -x'}_2, ~\forall  (x,x')\in \X^2\}$
  denotes the class of $k$-Lipschitz functions defined on $\X$
  and $\bigcup_{k\geq 0} \Lip(k)$  denotes the set of Lipschitz continuous functions.
  Finally,  $\mathcal{U}(\X)$ stands for
  the uniform distribution over a bounded measurable domain $\X$,
  $\mathcal{B}(p)$ for the Bernoulli distribution of parameter $p$,
  $\indic{\cdot}$ for the standard indicator function taking values in $\{0,1 \}$
  and the notation $X\sim \cal{P}$ means that the random variable $X$
  has the distribution $\cal{P}$.\\


  \subsection{Preliminary results}
  \label{sec:generic_results}
  
  In order to design efficient procedures,
    we first investigate the best performance that can be achieved
    by any algorithm over the class of Lipschitz functions.\\

  \noindent {\bf Sequential algorithms and optimization consistency.} 
  We first describe the sequential procedures that are considered here 
  and the corresponding concept of consistency in the sense of global optimization.
  
  \sloppy
    \begin{definition}
  \label{def:algo}
  {\sc (Sequential algorithm)} 
  The class of optimization algorithms
  we consider, denoted in the sequel by $\A$,
  contains all the algorithms $A=\{A_t\}_{t\geq 1}$
  completely described by:
  \begin{itemize}
    \item[1.] A distribution $A_1$ taking values in $\X$ which allows to generate
    the first evaluation point, {\it i.e.~}$X_1 \sim A_1$;
    \item[2.] An infinite collection of parametric distributions $\{A_t \}_{t\geq 2}$ 
    taking values in $\X$ and
    based on the previous evaluations 
    which define the iteration loop,
        {\it i.e.}~$X_{t+1} |X_1\dots,X_t
    \sim A_{t+1}( (X_1, f(X_1)),\dots, (X_t, f(X_t)))$.
  \end{itemize}
  \end{definition}
  Note that this class of algorithms also includes the deterministic methods
  in which case the distributions $\{ A_t\}_{t \geq 1}$ are degenerate.
  The next definition introduces the notion of asymptotic convergence.
  
      \begin{definition}
    \label{def:consistency}
  {\sc (Optimization Consistency)}
   A global optimization algorithm $A$ is said to be consistent over a set $\mathcal{F}$ of 
   real-valued functions admitting a maximum over their input domain $\X$ if and only if
   \[
    \forall f \in \mathcal{F},~\max_{i=1\dots n}f(X_i)
    \xrightarrow{p} \max_{x \in \X} f(x)
   \]
  where $X_1,\dots, X_n$ denotes a sequence of $n$ evaluations points generated 
  by the algorithm $A$  over the function $f$. 
  \end{definition}
  
  \noindent {\bf Asymptotic performance.} 
  We now investigate the minimal conditions 
  for a sequential algorithm to achieve asymptotic convergence.
  Of course, it is expected that  a global optimization algorithm 
  should be consistent at least for the class of Lipschitz functions
  and the following result reveals a necessary and sufficient condition (NSC)
  in this case.

  \begin{proposition}
  \label{prop:consistency_equivalence}
  {\sc (Consistency NSC)}
  A global optimization algorithm $A$ is consistent over 
  the set of Lipschitz functions if and only if
  \[
    \forall f \in \textstyle{\bigcup_{k \geq 0}\normalfont{\text{Lip}}(k)},~~
    \displaystyle{\sup_{x \in \X} \min_{i=1 \dots n} \lVert X_i - x\rVert_2
    \xrightarrow{p} 0.}
  \]
  \end{proposition}
  A crucial consequence of the latter proposition is that
  the design of any consistent method ends up to covering the whole input space
  regardless of the function values.
  The example below introduces the most popular space-filling method
  which will play a central role in our  analysis.

  \begin{example}
  \label{ex:PRS}
   {\sc (Pure Random Search)} The Pure Random Search {\sc (PRS)}
   consists in sequentially evaluating the function over a sequence 
   of points $X_1,X_2,X_3,\dots$
   uniformly and independently distributed over the input space $\X$.
   For this method, a simple union bound indicates that 
   for all $n\in\mathbb{N}^{\star}$ and $\delta \in (0,1)$,
   we have with probability at least $1-\delta$ and
   independently of the function values,
   $$
    \sup_{x \in \X} \min_{i=1 \dots n} \lVert X_i - x\rVert_2
    \leq  \diam{\X} \cdot
    \left( \frac{\ln(n/\delta)+d\ln(d)}{n} \right)^{\frac{1}{d}}\!\!\!.
   $$
  \end{example}
  In addition to this result, 
  we point out that the covering rate of any method 
  can easily be shown to be at best of order $\Omega(n^{-1/d})$
  and thus  subject to to the curse of dimensionality
  by means of covering arguments.
  Keeping in mind the equivalence of Proposition \ref{prop:consistency_equivalence}, 
  we may now turn to the nonasymptotic analysis.\\

  \noindent {\bf Finite-time performance.}
  We investigate here
  the best performance that can be achieved by any algorithm
  with a finite number of function evaluations. 
  We start by casting a negative result stating that any algorithm 
  can suffer, at any time, an arbitrarily large loss over the 
  class of Lipschitz functions.

  \begin{proposition}
    \label{prop:unbounded_error}
    Consider any global optimization algorithm $A$.
    Then, for any constant $C>0$ arbitrarily large,
    any $n\in\mathbb{N}^{\star}$ and $\delta \in (0,1)$, 
    there exists a function 
    $\tilde{f} \in\bigcup_{k\geq0}\Lip(k)$ only depending on $(A, C, n, \delta)$
    for which we have
    with probability at least $1-\delta$,
    \[
      C \leq \max_{x \in \X} \tilde{f}(x) - \max_{i=1 \dots n}\tilde{f}(X_i).
    \]
  \end{proposition}
  This result might not be very surprising since the class of Lipschitz functions
  includes functions with finite, but arbitrarily large variations.
  When considering the subclass of functions with fixed Lipschitz constant,
  it becomes possible to derive finite-time bounds 
  on the minimax rate of convergence.

  \begin{proposition} 
  \label{prop:minimax}
  {\sc (Minimax rate)}, adapted from \cite{bull2011convergence}.
  For any Lipschitz constant $k \geq 0$ 
  and any 
  $n \in \mathbb{N}^{\star}$, 
  the following inequalities hold true:
   \begin{align*}
     c_1 \cdot k \cdot n^{-\frac{1}{d}}\leq
     \inf_{A \in  \A}
     \sup_{f \in \normalfont{\text{Lip}}(k)}
     \esp{ \max_{x \in \X}f(x) - \max_{i=1\dots n}f(X_i) }
     \leq c_2 \cdot k \cdot n^{ -\frac{1}{d} }
   \end{align*}
   where $c_1 = \rad{\X}/(8\sqrt{d})$, $c_2 = \diam{\X} \times d!$
   and the expectation is taken over a sequence $X_1, \dots, X_n$ of
   $n$ evaluation points generated by the algorithm $A$ over $f$.
  \end{proposition}
  We stress that this minimax convergence rate of order $\Theta(n^{-1/d})$
  can still be achieved by any method with
  an optimal covering rate of order $O(n^{-1/d})$. Observe indeed that since
  $
   \esp{\max_{x \in \X}f(x) - \max_{i=1\dots n}f(X_i)}
   \leq k\times \esp{ \sup_{x \in \X} \min_{i=1\dots n}\norm{x - X_i}_2 }
  $
  for all  $f \in \Lip(k)$,~
  then an optimal covering rate necessarily implies minimax efficiency.
  However, as it can be seen by examining the proof of Proposition \ref{prop:minimax}
  provided in the Appendix Section,
  the functions constructed to prove the limiting bound of $\Omega(n^{-1/d})$
  are spikes which are almost constant  everywhere 
  and  do not present a large interest from a practical perspective.
  In particular, we will see in the sequel that  one can 
  design:
  \begin{itemize}
   \item[I)] An  algorithm with fixed constant $k\!\geq\!0$
   which achieves minimax efficiency
   and also presents exponentially decreasing rates
  over a large subset of functions, as opposed to space-filling	 methods
  (LIPO, Section \ref{sec:lipopt}).
  \item[II)] A consistent algorithm which does not require the knowledge of the Lipschitz
  constant and presents comparable performance as
  when the constant $k$ is assumed to be known (AdaLIPO, Section \ref{sec:adalipopt}).
  \end{itemize}


  \section{Optimization with fixed Lipschitz constant}
  \label{sec:lipopt}
  
  In this section, we consider the problem of optimizing 
  an unknown function $f$ given the knowledge that
  $f \in \text{Lip}(k)$ for a given $k\geq 0$.

  \subsection{The LIPO Algorithm}

\newcommand{\nosemic}{\renewcommand{\@endalgocfline}{\relax}}
\newcommand{\dosemic}{\renewcommand{\@endalgocfline}{\algocf@endline}}
\newcommand{\pushline}{\Indp}
\newcommand{\popline}{\Indm\dosemic}
\let\oldnl\nl
\newcommand{\nonl}{\renewcommand{\nl}{\let\nl\oldnl}}
\SetNlSty{textbf}{}{.}

\RestyleAlgo{boxed}
\begin{figure}[t!]
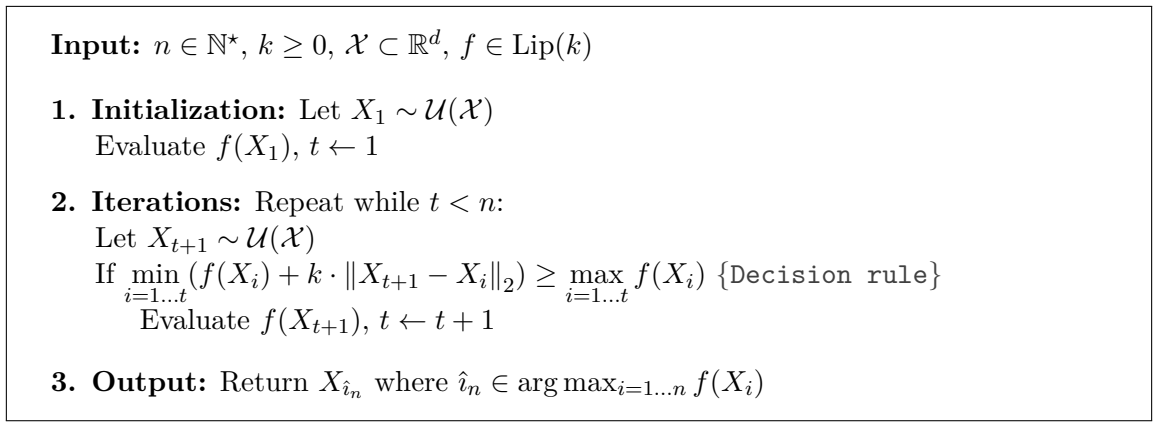

  \begin{algorithm}[H]
  \vspace{0.5em}
  \textbf{Input:} $n\in\mathbb{N}^{\star}$, $k\geq0$, $\X\subset\R^d$, $f\in\Lip(k)$\\
  \vspace{1em}
  \textbf{1. Initialization:} Let $X_1 \sim \mathcal{U}(\X)$\\
  \nonl \pushline  Evaluate $f(X_1)$, $t \leftarrow 1$\\
  \vspace{0.7em}
  \nl \Indm \textbf{2. Iterations:} Repeat while $t <n$: \\
  \pushline \nonl Let $X_{t+1} \sim \mathcal{U}(\X)$\\
  \nonl If  $\displaystyle{\min_{i=1\dots t} (f(X_i) 
	+ k \cdot \norm{X_{t+1}-X_i}_2 )
	\geq \max_{i=1\dots t} f(X_i)}$~\textcolor{black!75}{\tt \{Decision rule\}}\\
  \pushline \nonl Evaluate $f(X_{t+1})$, $t \leftarrow t+1$\\\vspace{0.2em}
  \vspace{0.7em}
  \Indm \Indm \textbf{3. Output:} Return $X_{\hatin}$ where 
  $\hatin \in \arg\max_{i=1 \ldots n}f(X_i)$
  \vspace{0.5em}
  \end{algorithm}
  \vspace{-0.5em}
  \caption{The  LIPO algorithm}
  \label{fig:rankopt}
\end{figure}

  The inputs of the LIPO algorithm 
  (displayed in Figure \ref{fig:rankopt}) are a number $n$ of function evaluations,
  a Lipschitz constant $k\geq0$, the input space $\X$ and the unknown function $f\in\Lip(k)$.
  At each iteration $t \geq 1$, a random variable
  $X_{t+1}$ is sampled uniformly over the input space $\X$ and the algorithm 
  decides whether or not to evaluate the function at this point.
  Indeed, it evaluates the function over $X_{t+1}$ if and only if the value 
  of the upper bound on possible values
  $UB_{k,t}:x \mapsto \min_{i=1\dots t} f(X_i)$ $+ k$ $\cdot\norm{x - X_i}_2$
  evaluated at this point and
  computed from the previous evaluations,
  is at least equal to the value of the 
  best evaluation observed so far $\max_{i=1\dots t}f(X_i)$.
  To illustrate how the decision rule operates in practice,
  an example of the computation of the upper bound can be found in Figure \ref{fig:ub}.\\

  \noindent More formally, the mechanism behind 
  the decision rule can be explained
  using the active subset of consistent functions previously considered in active learning
  (see, {\it e.g.}, \cite{dasgupta2011two} or \cite{hanneke2011rates}).
  \begin{definition}
  \label{def:consistent_functions}
   {\sc (Consistent functions)}
   The active subset of $k$-Lipschitz functions 
   consistent with the unknown function $f$ over a sample 
   $(X_1,f(X_1)), \dots,(X_t,f(X_t))$
  of $t\geq1$ evaluations is defined as follows:
     \[
    \mathcal{F}_{k,t} := \{ g \in \Lip(k) : \forall i \in\{ 1 \dots t\},
    ~ g(X_i) = f(X_i)\}.
  \]
  \end{definition}
  One can indeed recover from this definition
  the subset of points which can actually maximize the target function $f$.
    \begin{definition}
    \label{def:potential}
   {\sc (Potential maximizers)}
   Using the same notations as in Definition \ref{def:consistent_functions},
   we define the subset of potential maximizers 
   estimated over any sample $t\geq 1$ evaluations
   with a constant $k\geq0$
   as follows:
  \[
   \X_{k,t} := \left\{x \in \X: \exists g \in \F_{k,t} \normalfont{\text{~\!~such that~~\!}} x 
   \in \underset{x \in \X}{\arg \max}~g(x) \right\}\!\!.
  \]
  \end{definition}
  We may now provide an equivalence which makes
  the link with the decision rule of the LIPO algorithm.
    \begin{lemma}
  \label{lem:potential}
   If  $\X_{k,t}$ denotes the set of potential maximizers defined above,
   then we have the following equivalence:
   \[
    x \in \X_{k,t} \Leftrightarrow \min_{i=1\dots t}  f(X_i) + k \cdot \norm{x-X_i}_2
    \geq \max_{i=1\dots t}f(X_i).
   \]
  \end{lemma}
  Hence, we deduce from this lemma 
  that the algorithm only evaluates the function
  over points that still have a chance to be a maximizer of the unknown function.

\begin{figure}[!t]
\begin{center}
\includegraphics[height = 2.3cm, width=4cm ]{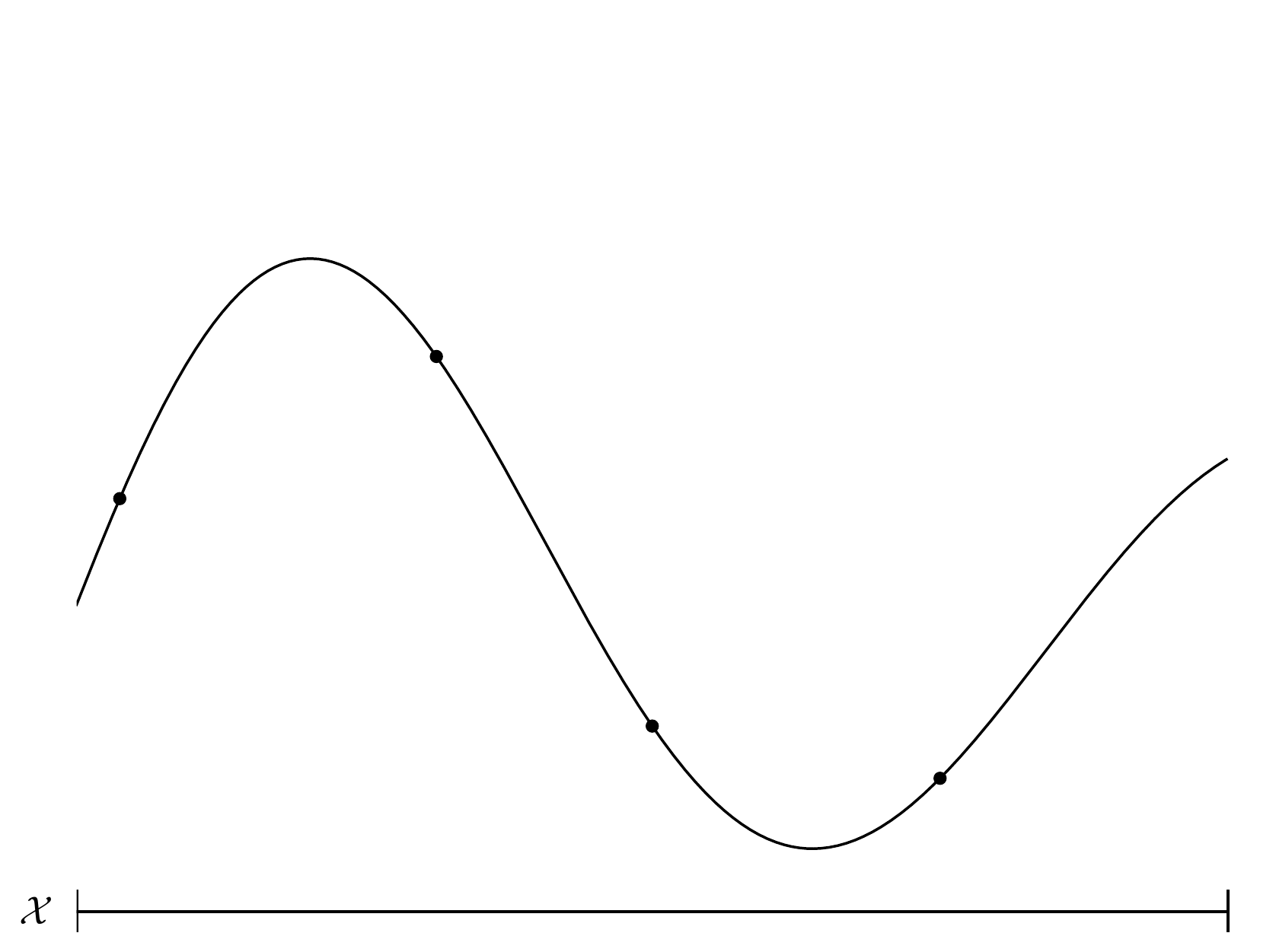}  
~~~~
\includegraphics[height = 2.3cm, width=4cm]{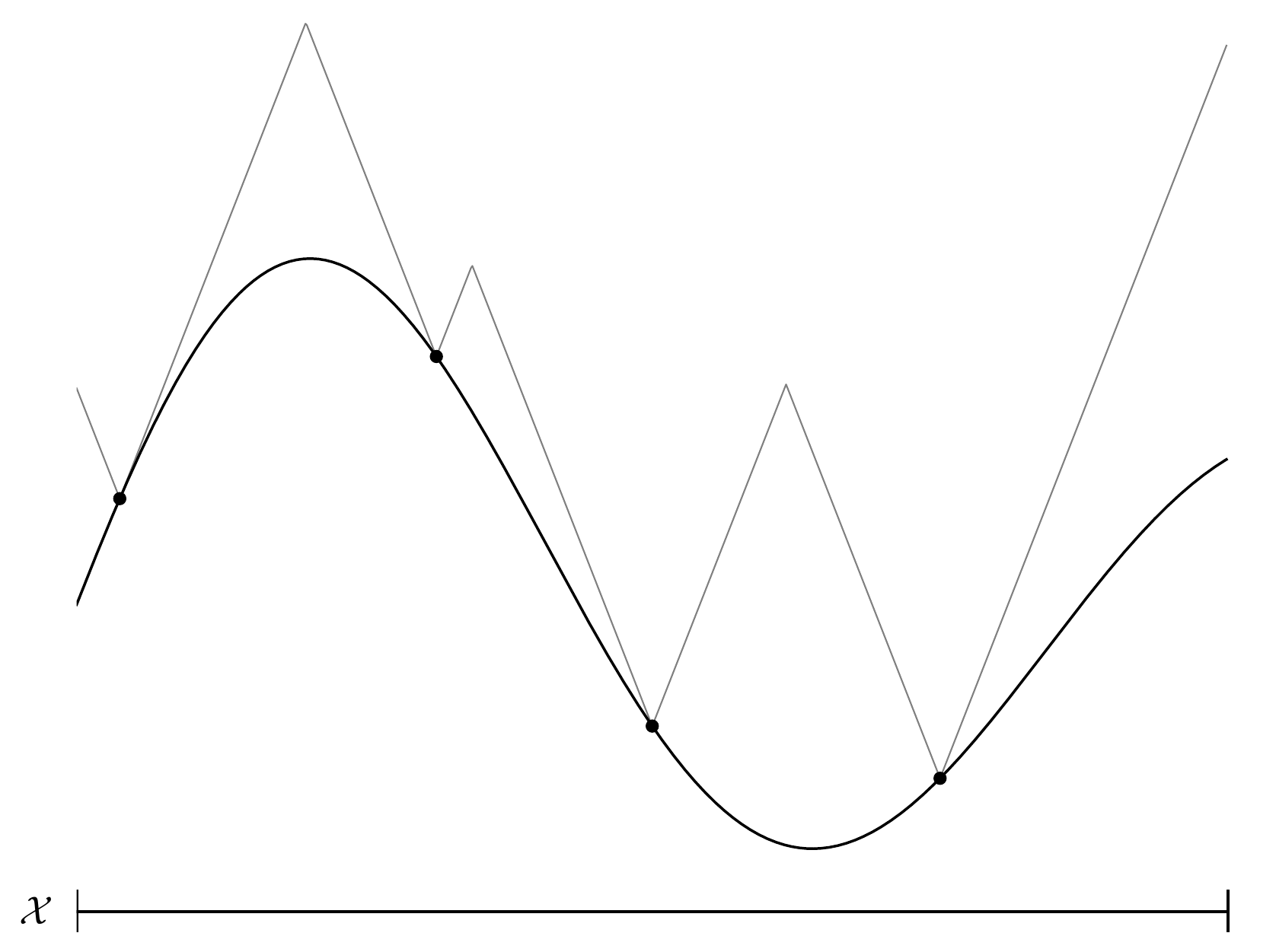}
~~~~
\includegraphics[height = 2.3cm, width=4cm]{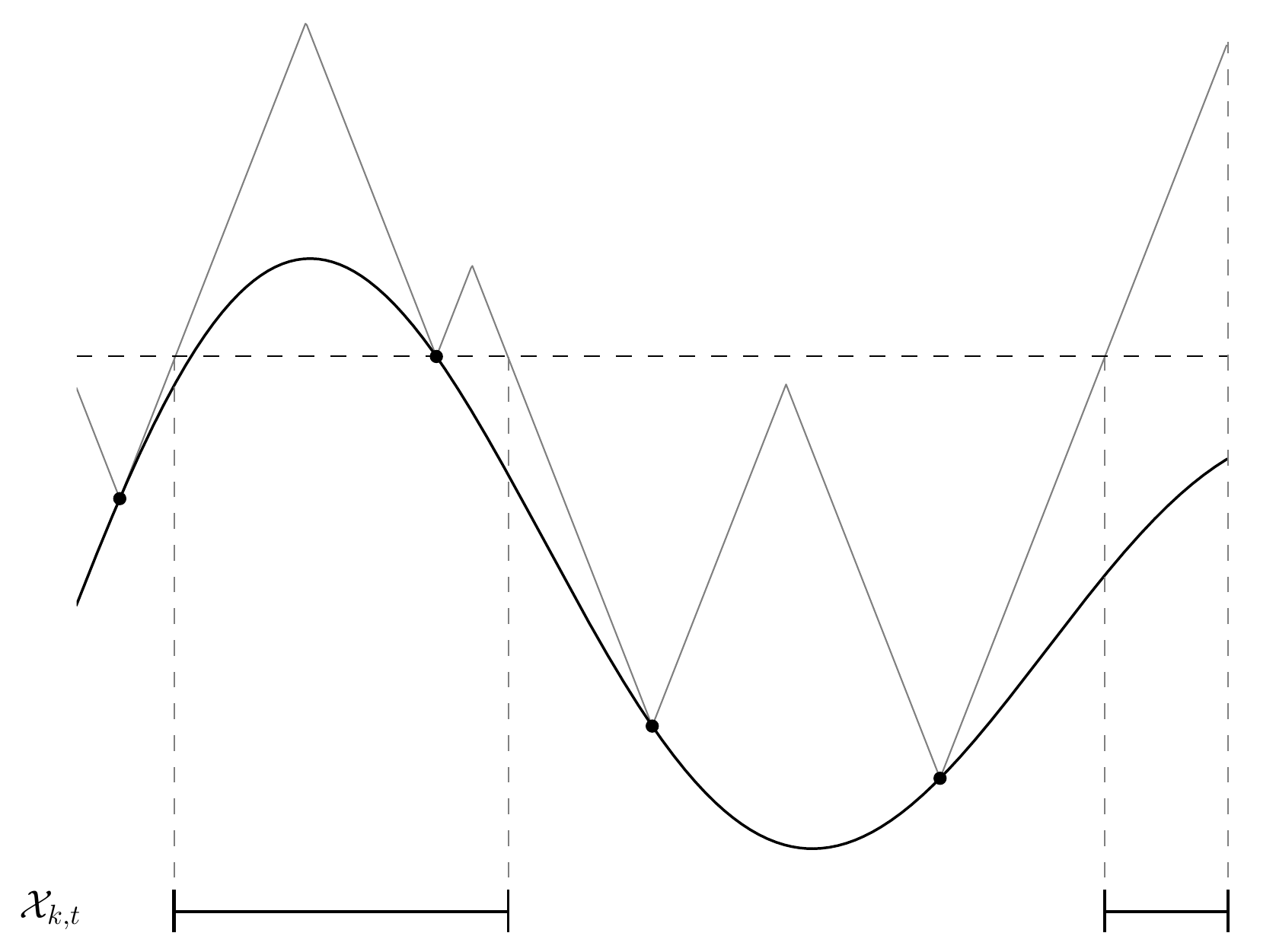}
\end{center}
\vspace{-1.5em}
\caption{{\it Left:}\! A Lipschitz function and
a sample of $t=4$ evaluations.
{\it Middle:}
In grey, the upper bound 
$UB: x \mapsto \min_{i=1\dots t} f(X_i)$ $+ k$ $\cdot\norm{x - X_i}_2$.
{\it Right:} the set of points   $\X_{k,t} :=\{x \in \X: UB(x) \geq \max_{i=1\dots t}f(X_i) \}$
 which satisfy the decision rule.
}
\label{fig:ub}
\end{figure}


    \begin{remark}{\sc (Adaptation to noisy evaluations)}
  \label{rem:noisy}
  In addition to these defintions, we point out that 
  the LIPO algorithm could
  be extended to settings with noisy evaluations 
  by slightly adapting the ideas developed in \cite{dasgupta2011two} and \cite{hanneke2011rates}.
  More specifically, when considering a sample 
  $(X_1, Y_1), \dots, (X_n,Y_n)$
  of $n\geq 1$ noisy observations where $Y_i = f(X_i) + \sigma \epsilon_i$
  and $\epsilon_i\iid \mathcal{N}(0,1)$ 
  and observing that the empirical mean-squared error
  $R_n(f) := (1/n)\sum_{i=1}^n (f(X_i) -Y_i )^2 = 
  (\sigma^2/n)\sum_{i=1}^n\epsilon_i^2$
  evaluated in $f$ is distributed as a chi-square,
  a possible approach would consist in using a relaxed
  version of the active subset $\F_{k,\delta, t} := \{g \in \Lip(k):
  R_n(g) \leq (\sigma^2/n)\cdot \chi^2_{1-\delta, n} \}$
  of Definition \ref{def:consistent_functions}
  where $\chi^2_{1-\delta, n}$  denotes the $1-\delta$ 
  quantile of the chi-squared distribution with $n$ degrees of freedom.
  \end{remark}

    \begin{remark}
     {\sc (Extension to other smoothness assumptions)} 
     Additionally, it is also important to note 
     the proposed optimization scheme 
     could easily be extended to a large number of
     classes of globally and locally smooth functions by slightly adapting the 
     decision rule. 
     For instance, when
     $\F_{\ell}=\{f :\X\to \R \mid x^{\star}\text{~is unique~and~}
     ~\forall x \in \X, f(x^{\star})-f(x) \leq \ell(x^{\star},x)\}$ 
     denotes the set of functions
     previously considered in \cite{munosmono} which are
     locally smooth around their maxima
     with regards to a given semi-metric $\ell:\X\times\X \to \R^+$\!\!,
     a straightforward derivation of Lemma \ref{lem:potential} directly gives that
     the decision rule applied in $X_{t+1}$ would simply consists in testing whether
     $\max_{i=1\dots t}f(X_i) \leq \min_{i=1\dots t}f(X_i) +\ell(X_{t+1},X_i)$.
     However, since the purpose of this work is to design fast algorithms for Lipschitz
     functions, we will only derive convergence results
     for the version of the algorithm stated above.
    \end{remark}

  \subsection{Convergence analysis}  
  We start by casting the consistency property of the algorithm.

  \begin{proposition}
  \label{prop:consistency_lipopt}
   {\sc (Consistency)} 
   For any Lipschitz constant $k\geq 0$,
   the LIPO algorithm tuned with a parameter $k$ 
   is consistent over the set $k$-Lipschitz
   functions, {\it i.e.}
   \[
    \forall f \in \normalfont{\text{Lip}}(k),
    ~\max_{i=1\dots n} f(X_i) \xrightarrow{p} \max_{x \in \X} f(x).
   \]
  \end{proposition}
  The next result shows that the value of the highest evaluation observed by 
  the algorithm is always superior or equal in the usual stochastic ordering sense
  to the one of a Pure Random Search.
  
  \begin{proposition}
  \label{prop:fasterprs}
   {\sc (Faster than pure random search)} 
   Consider the LIPO algorithm tuned with any constant $k\geq0$.
   Then, for any $f \in \normalfont{\text{Lip}}(k)$ and
   $n \in \mathbb{N}^{\star}$,
   we have that $\forall y \in \R$,
   \[
    \P\left(\max_{i=1\dots n } f(X_i) \geq y \right) \geq 
    \P\left(\max_{i=1\dots n } f(X'_i) \geq y \right)
   \]
   where $X_1,\dots, X_n$ is a sequence of $n$ evaluation points generated by LIPO
  and $X_1',\dots, X'_n$ is a sequence of $n$ 
  independent random variables uniformly distributed over $\X$.
  \end{proposition}
  Based on this result, one can easily derive a first finite-time bound on the difference
   between the value of the true maximum and its approximation.
  \begin{corollary}
  \label{prop:upperlipopt}
   {\sc (Upper bound)}
   For any $f\in \normalfont{\text{Lip}}(k)$,
   any $n \in \mathbb{N}^{\star}$ and
   $\delta \in (0,1)$, we have with probability at least $1-\delta$,
   \[
    \max_{x \in \X}f(x) - \max_{i=1\dots n} f(X_i) 
    \leq k \cdot \diam{\X} \cdot \left( \frac{\ln(1/\delta)}{n} \right)^{\frac{1}{d}}\!\!.
   \]
  \end{corollary}
  This bound which proves the miminax optimality of LIPO
  stated in Proposition \ref{prop:minimax} once integrated does
  however not show any improvement over PRS
  and it cannot be significantly improved without making any additional assumption
  as shown below.

  \begin{proposition}
  \label{prop:limit_lipopt}
   For any $n\in\mathbb{N}^{\star}$ and $\delta \in (0,1)$,
   there exists a function $\tilde{f} \in \normalfont{\text{Lip}}(k)$,
   only depending on $n$ and $\delta$,
   for which we have with probability at least $1-\delta$:
   \[
     k \cdot \rad{\X} \cdot \left( \frac{\delta}{n}\right)^{\frac{1}{d}} 
     \leq  \max_{x \in \X}\tilde{f}(x) - \max_{i=1 \dots n}\tilde{f}(X_i).
   \]
  \end{proposition}
  As announced in Section \ref{sec:generic_results}, one can nonetheless 
  get tighter polynomial bounds and even an exponential decay
  by using the following condition 
  which describes the behavior of the function around its maximum.

  \begin{condition}
  \label{cond}
  {\sc (Decreasing rate around the maximum)}
   A function $f:\X \to \R$ is $(\kappa, c_{\kappa})$-decreasing around its maximum
   for some $\kappa \geq 0$, $c_{\kappa} \geq 0$ if:
   \begin{enumerate}
    \item The global optimizer $x^{\star}\in \X$ is unique;
    \item For all $x \in \X$, we have that:
    \[
      f(x^{\star}) - f(x) \geq  c_{\kappa} \cdot \norm{x-x^{\star}}_2^{\kappa}.
    \]
   \end{enumerate}
  \end{condition}
  This condition, already considered in the works of \cite{zhigljavsky1991theory} 
  and \cite{munosmono}, captures how fast the function decreases around its maximum.
  It can be seen as a local one-sided H\"{o}lder condition 
  that can only be met for $\kappa\geq 1$
  when $f$ is assumed to be Lipschitz.
  As an example,  three functions satisfying this condition
  with different values of $\kappa$ are displayed 
  on Figure \ref{fig:smooth}.

       \begin{figure}[!b]
    \begin{center}$
      \begin{array}{ccc}
      \includegraphics[width=28mm, height=17.5mm]{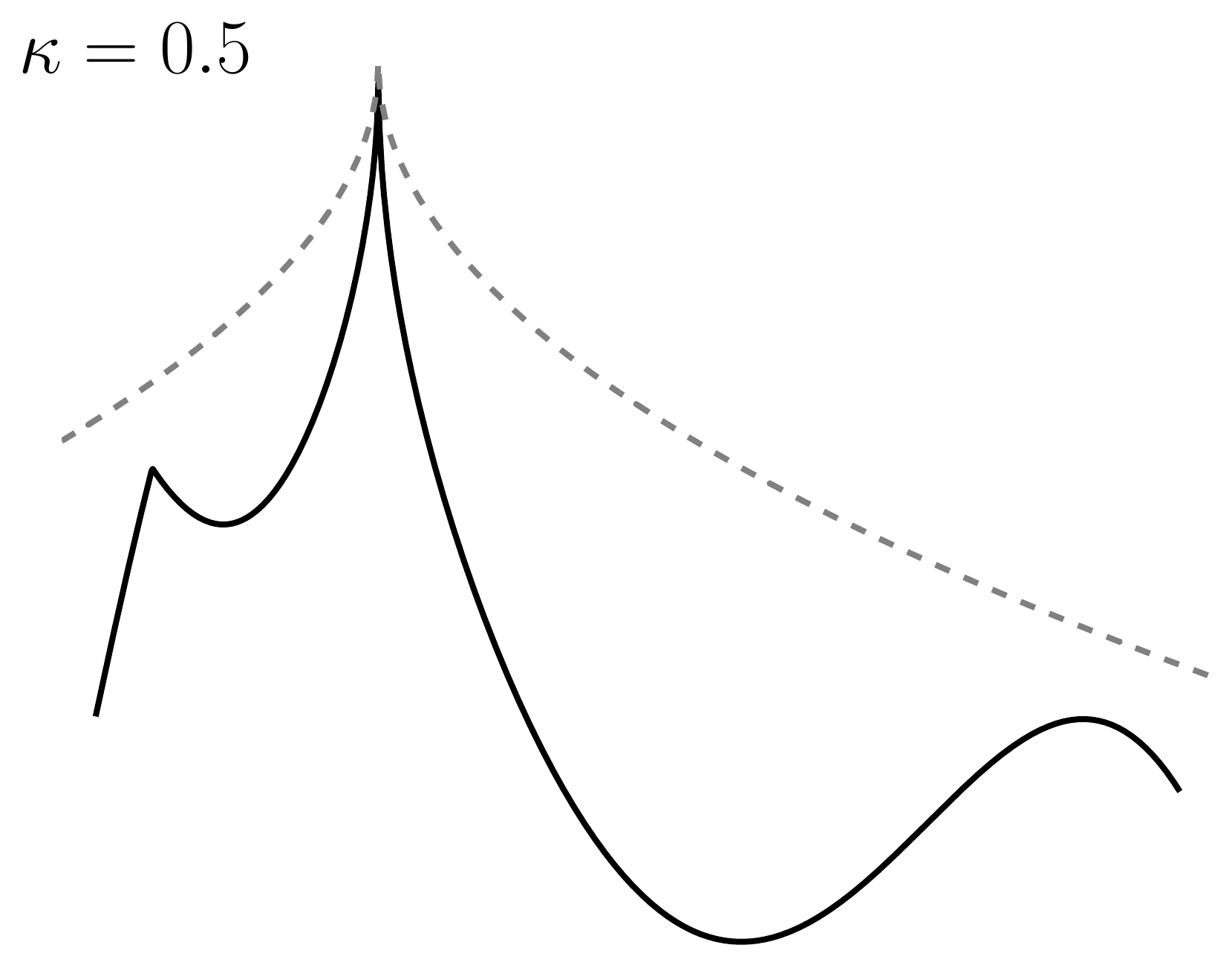}&~~~~~~~
      \includegraphics[width=28mm, height=17.5mm]{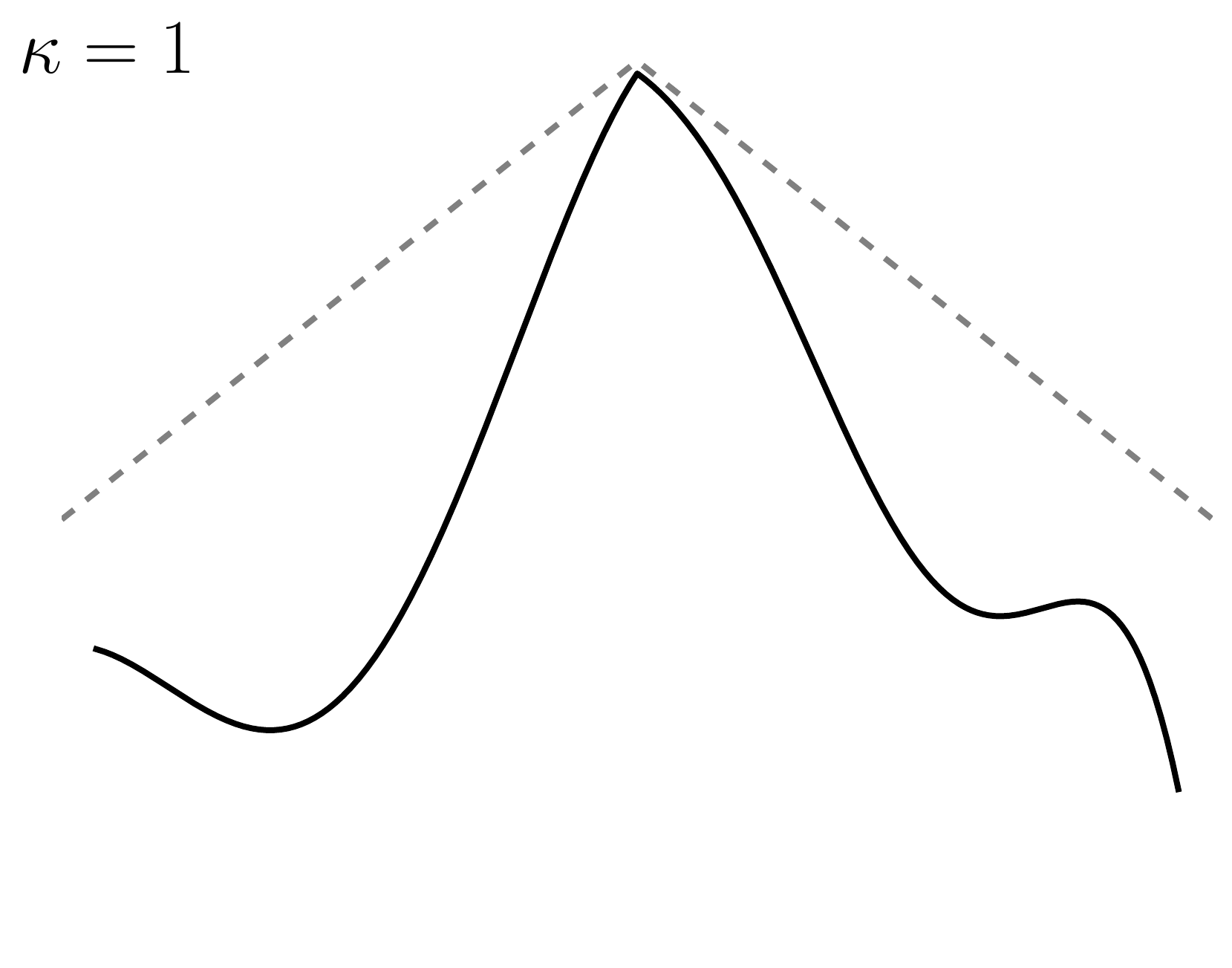}&~~~~~~~
      \includegraphics[width=28mm, height=17.5mm]{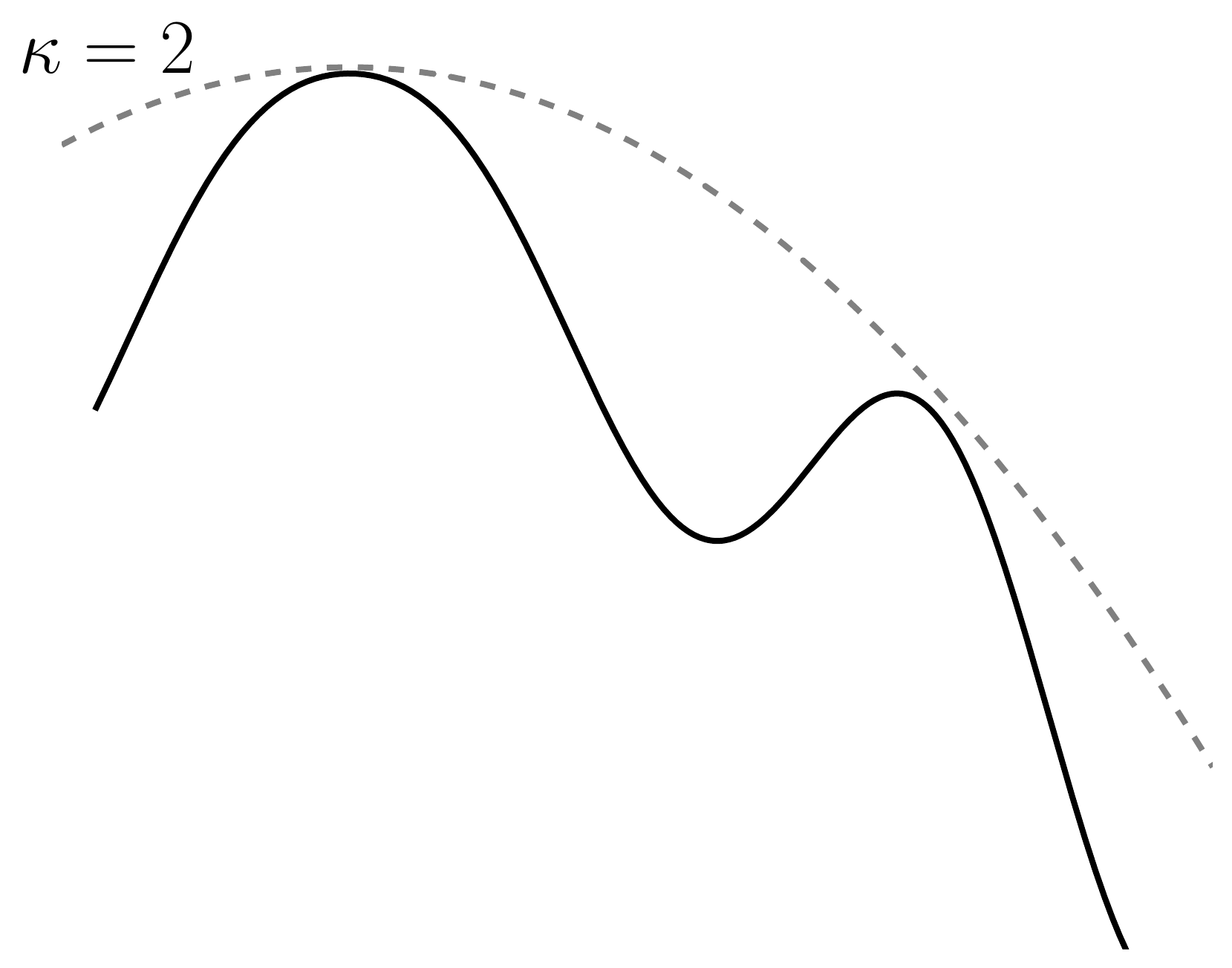}
      \end{array}$
    \end{center}
    \vspace{-1.5em}
    \label{fig:smooth}
   \caption{Three one-dimensional functions satisfying Condition \ref{cond}
   with $\kappa =1/2$ ({\it Left}), $\kappa = 1$ ({\it Middle}) and
   $\kappa = 2$ ({\it Right}).}\label{fig:exh_prs}
  \end{figure}

  \begin{theorem}
  \label{th:fast_rates}
  {\sc (Fast rates)}
   Let $f\in\Lip(k)$ be any Lipschitz function 
   satisfying Condition \ref{cond} for some $\kappa \geq 1, c_{\kappa} > 0$.
   Then, for any $n\in\mathbb{N}^{\star}$ and $\delta \in (0,1)$, we have 
   with probability at least $1-\delta$,
    \[    
    \max_{x \in \X}f(x) -\max_{i=1\dots n} f(X_i) \leq k \cdot \diam{\X} \times
    \begin{dcases*}
     \exp\left\{-~  C_{k, \kappa} \cdot 
	\frac{n\ln(2)}{\ln(n/\delta) + 2(2\sqrt{d})^d }
	\right\},   \text{~~~~~~~~}\  \kappa =1,  \\
	~\\
      \frac{2^{\kappa}}{2}
	\left( 1 + C_{k, \kappa}
     \cdot  \frac{n(2^{d(\kappa~\!\minus~\!1)}-1)}{\ln(n/\delta) +2(2\sqrt{d})^d}
	\right)^{-\frac{\kappa}{d(\kappa-1)}}\!\!\!\!\!\!\!\!\!\!\!\!\!
	\!\!\!\!\!\!\text{~,~}
	~~~~~~~\kappa>1,
    \end{dcases*}
   \]
  where $C_{k,\kappa}=(c_{\kappa} \max_{x\in\X}\norm{x-x^{\star}}^{\kappa-1}/ 8k)^d $.
  \end{theorem}
    We point out that the polynomial bound can be slightly improved
    and shown to be of order 
    $O_{\P}^{*}(n^{-\frac{\kappa \times \kappa}{d(\kappa -1)}})$
    in the case where the function is 
    locally equivalent to $\norm{x^{\star}-x}_2^{\kappa}$
    ({\it i.e.,} when $\exists c_{\kappa},c_2>0$, 
    $c_{\kappa} \norm{x^{\star}-x}_2^{\kappa} 
    \leq f(x^{\star}) -f(x) \leq c_2 \norm{x^{\star}-x}_2^{\kappa}$).
    The last result we provide states an exponentially decreasing lower bound.
  \begin{theorem}
  \label{th:lower_bound}
   {\sc (Lower Bound)}
   For any $f\in\Lip(k)$ 
   satisfying Condition \ref{cond} for some $\kappa \geq 1, c_{\kappa} > 0$ and
   any $n\in\mathbb{N}^{\star}$ and $\delta \in (0,1)$, we have 
   with probability at least $1-\delta$,
   \begin{align*}
    c_{\kappa} \cdot 
    \rad{\X}^{\kappa} \cdot 
    e^{ -\frac{\kappa}{d} \left( n + \sqrt{2n\ln(1/\delta)}  + \ln(1/\delta) \right)}
    \leq \max_{x \in \X}f(x) - \max_{i=1\dots n}f(X_i).
   \end{align*}
  \end{theorem}
  The next section provides 
  an explicit derivation of the fast rate on some toy examples
  and a discussion on these results can be found in Section \ref{sec:compareason_lipo}
  where LIPO is compared with similar algorithms.

  \subsection{Examples}
  \label{sec:examples}
  The next examples consider that the optimization is performed
  over the hypercube $\X =[-R,R]^d$ for some fixed $R>0$ and $d\geq 1$.\\

    \noindent {\bf Sphere function.} The sphere function $f(x) = 1 - \norm{x}_2$ 
    is the canonical example of Lipschitz function. 
    For this function, 
    the Lipschitz continuity is a direct consequence of the triangle inequaliy: 
    $\forall (x,y) \in \X^2$,  we have that
    $
     \abs{f(x) - f(y)} = \norm{x}_2 - \norm{y}_2= \norm{y + x- y }_2  - \norm{y}_2
     \leq  \norm{x-y}_2
    $
    by assuming w.l.og.~that $\norm{x}_2\geq \norm{y}_2$.
    Observing  now that  $x^{\star} = \vec{0}$ and
    $
     f(x^{\star})  - f(x) =   \norm{x^{\star} -x }_2
    $
    for all $x \in \X$, it is easy to see that
    Condition \ref{cond} is satisfied with $\kappa = 1$ and $c_{\kappa} = 1$.
    Hence, running LIPO tuned with any $k\geq1$ would provide 
    an exponenitally decreasing rate of order $O^{*}_{\P}(e^{- n/2(16k\sqrt{d})^d })$.\\

    \noindent {\bf Linear slope.}
    The second class of functions we consider are the linear functions
    of the form $f(x) = 1-\inner{w,x}$ with weight vectors $w \in \R^d$.
    Applying Cauchy-Schwartz inequality directly gives the Lipschitz continuity: 
    $\forall (x,y)\in \X^2$,
     $
     \abs{f(x) - f(y)} =\abs{  \inner{w, x-y} }
     \leq \norm{w}_2 \cdot \norm{x-y}_2.
    $
    In the case of non-zero weights,
    $x^{\star} = -R\times(\sgn{w_1},\dots, \sgn{w_d})$ and 
    we have for all $x \in \X$,
    $
     f(x^{\star} ) - f(x)  = \inner{ w, x- x^{\star}}
     \geq \min_{i=1\dots d}|w_i| \cdot  \norm{x^{\star} - x}_2.
    $
    Therefore, Condition \ref{cond} is satisfied with $\kappa=1$ and 
   $c_{\kappa} = \min_{i=1\dots d} |w_i|$ and we deduce that
   running LIPO with any $k\geq \norm{w}_2$ would provide 
   a decay of order $O^{*}_{\P}(e^{- n/2(16k\sqrt{d} 
   \min|w_i|)^d })$.\\

    \noindent {\bf Largest coordinate.}
    The last function we consider is the largest coordinate function
    $f(x) =  1 - \max\{|x_1|,\dots, |x_d|\}$. 
    Taking any $(x,y) \in \X^2$
    and denoting by $i(x)$ and $i(y)$ the indexes of (one of)
    their largest absolute coordinate we obtain
    that
    $
     \abs{f(x) -f(y) } = |x_{i(x)}| - |y_{i(y)}| \leq 
     |x_{i(x)}| -|y_{i(x)}|
     \leq |x_{i(x)} -y_{i(x)}|
     \leq \norm{x-y}_2
    $
    by assuming w.l.o.g.~that $x_{i(x)}\geq y_{i(y)}$.
    Noticing that $x^{\star}=\vec{0}$
    and that
    $
     f(x^{\star}) - f(x) = \max_{i=1\dots d}|x_i| 
     \geq \norm{x^{\star} - x}_2/\sqrt{d}
    $
    for all $x \in \X$,
    we deduce that Condition \ref{cond} is satisfied with 
    $\kappa =1$ and $c_{\kappa} =1/\sqrt{d}$.
    Thus, running LIPO with any $k\geq 1$
    would provide an exponential decay of order $O^{*}_{\P}(e^{- n/2(16kd)^d })$.
  
  \noindent Note that the polynomial bounds could also be derived from the previous examples
  by slightly adapting the functions. For instance, 
  it is easy to see that the function $f(x) =1- \norm{x}_2^{\kappa}$ 
  satifies Condition \ref{cond} with $\kappa\geq0$ and $c_{\kappa}=1$.

  \subsection{Comparison with existing works}
  \label{sec:compareason_lipo}
  
    {\bf Algorithms.} The Piyavskii algorithm  (\cite{piyavskii1972algorithm})
    is a sequential algorithm which requires the knowledge of the
    Lipschitz constant $k\geq 0$ and iteratively evaluates the function over a point
    $X_{t+1} \in \arg\max_{x \in \X} UB_{k,t}(x)$
    which maximizes the upper bound on possible values
    $UB_{k,t}(x) :=\min_{i=1\dots t} f(X_i) + k\cdot \norm{x-X_i}$.
    \cite{munosmono} also proposed a similar algorithm (DOO)
    that requires the knowledge of a semi-metric 
    $\ell: \X\times\X \to \R^+$ for which the function is at least locally smooth
    ({\it i.e.,~}$\forall x \in \X$, $f(x^{\star}) -f(x) \leq \ell(x, x^{\star})$)
    and a hierarchical partitioning of the space $\X$
    in order to sequentially expand and evaluate the function over the center of a partition
    which has recorded the highest upper bound computed according to the semi-metric $\ell$.
    With regards to those works, the LIPO algorithm
    aims at optimizing globally Lipschitz functions and
    combines space-filling and exploitation rather than pure exploitation.
    Recall indeed that LIPO evaluates the function over
    any point which might improve the function values (see Lemma \ref{lem:potential})
    while DOO and the Piyavskii algorithm sequentially select among a restricted set of points
    the next evaluation point which have recorded the highest upper bound
    on possible values.\\

    \noindent {\bf Results.}
    To the best of our knowledge, only the consistency of the Piyavskii algorithm
    was proven in \cite{mladineo1986algorithm} and \cite{munosmono} derived 
    finite-time upper bounds for DOO 
    with the use of weaker local smoothness assumptions.
    To cast their results into our framework, we thus considered
    DOO  tuned with the semi-metric  $\ell(x,x')= k\cdot \norm{x-x'}_2$ over 
    the domain $\X=[0,1]^d$ partitioned into a $2^d$-ary tree of hypercubes
    and with  $f$ belonging to the sets of globally smooth functions:
   (a) $\Lip(k)$, (b) $\mathcal{F}_{\kappa}
  \!\!=\! \{ f \in \Lip(k) \text{~satisfying Condition \ref{cond} with~}
  c_{\kappa}, \kappa \geq 1\}$
  and (c) $\mathcal{F}'_{\kappa} = \{f \in \mathcal{F}_{\kappa}: 
  \exists c_2>0,~f(x^{\star}) -f(x) \leq c_2 \norm{x - x^{\star}}_2^{\kappa}\}$.
  The results of the comparison can be found in Table \ref{tab:comparisonlipopt}.
  In addition to the novel lower bounds and the rate over $\Lip(k)$, 
  we were able to obtain similar upper bounds as DOO over $\mathcal{F}_{\kappa}$,
  uniformly better rates for the functions in $\F'_{\kappa}$
  locally equivalent to $\norm{x^{\star}-x}_2^{\kappa}$ with $\kappa>1$
  and up to a constant factor a similar exponenital rate when $\kappa=1$ .
  Hence, when $f$ is only known to be $k$-Lipschitz,
  one thus should expect the algorithm exploiting the global smoothness (LIPO)
  to perform asymptotically better or at least similarly
  to the one using the local smoothness (DOO) or no information (PRS).
  However, keeping in mind that the constants are not necessarily optimal,
  it is also interesting to note that the term $(k\sqrt{d}/c_{\kappa})^d$ 
  appearing in both the fast rates of LIPO and DOO tends to suggest that if $f$ is 
  also known to be locally smooth for some $k_{\ell}\! \ll\! k$,
  then one should expect the algorithm exploiting the local smoothness $k_{\ell}$ 
  to be asymptotically faster than
  the one using the global smoothness $k$ in the case where $\kappa=1$.
  

     \begin{table}[!h]
     \vspace{0.5em}
  \centering
  {\small
  \begin{tabular}{@{}lccccc@{}}
  \toprule
  \textbf{Algorithm} & {\textbf{DOO}} & \textbf{LIPO} & 
  \textbf{Piyavskii} & \textbf{PRS} \\ 
  \midrule
  $f\in\Lip(k)$ &       &      &   &    \\
  Consistency &    \checkmark       &        \checkmark         &      \checkmark
  &      \checkmark       \\[-.5ex]
  Upper Bound &  -   &   $O_{\P}(n^{-\frac{1}{d}})$  &   - &       $O_{\P}(n^{-\frac{1}{d}})$       \\
  &&&& 
  \\[-3ex]
  \\\cmidrule[0.5\lightrulewidth]{2-5}
  $f \in \F_{\kappa}$, $\kappa>1$  &&&& \\[-0.5ex]
  Upper bound & $O(n^{-\frac{ \kappa}{d(\kappa~\!\minus~\!1)}})$ & 
  $O^{*}_{\P}(n^{-\frac{ \kappa}{d(\kappa~\!\minus~\!1)}})$& - & $O_{\P}(n^{-\frac{1}{d}})$\\ 
  Lower bound & - &  $\Omega^{*}_{\P}(e^{- \frac{\kappa}{d}n})$ & - & $\Omega_{\P}(n^{-\frac{\kappa}{d}})$ \\ 
  &&&& 
  \\[-3ex]
  \\\cmidrule[0.5\lightrulewidth]{2-5}
  $f \in \F'_{\kappa}$, $\kappa>1$  &&&& \\[-0.5ex]
  Upper bound 
  & $O(n^{-\frac{ \kappa}{d(\kappa~\!\minus~\!1)}})$   
  & $O^{*}_{\P}(n^{-\frac{ \kappa \times \kappa}{d(\kappa~\!\minus~\!1)}})$
  & - 
  & $O_{\P}(n^{-\frac{\kappa}{d}})$\\ 
  Lower bound &  -   & $\Omega^{*}_{\P}(e^{- \frac{\kappa}{d}n})$ & - & $\Omega_{\P}(n^{-\frac{\kappa}{d}})$\\ 
    \\[-3ex]
  \\\cmidrule[0.5\lightrulewidth]{2-5}
  $f \in \F'_{\kappa}$, $\kappa=1$ &&&& \\[-0.5ex]
  Upper bound &   
  $O(e^{ -\frac{n \ln(2)}{(2k\sqrt{d}/c_{\kappa})^d} })$ 
  &  
  $O_{\P}^{*}(e^{ -\frac{n \ln(2)}{2(16k\sqrt{d}/c_{\kappa})^d} })$ 
  & - & $O_{\P}(n^{-\frac{1}{d}})$\\ 
  Lower bound &  -   & $\Omega^{*}_{\P}(e^{- \frac{n}{d}})$ & - & $\Omega_{\P}(n^{-\frac{1}{d}})$\\ 
  \bottomrule
  \end{tabular}
  \vspace{-0.2em}
  \caption{Comparison of the results reported over the difference
  $\max_{x \in \X}f(x) - \max_{i=1\dots n}f(X_i)$ 
  in global optimization literature.
  Dash symbols are used when no results could be found.}\label{tab:comparisonlipopt}
  }
  \end{table}

  
\section{Optimization with unknown Lipschitz constant}
  \label{sec:adalipopt}
  
  In this section, we consider the problem of optimizing 
  any unknown function $f$ in the class $\bigcup_{k \geq 0} \text{Lip}(k)$.
  
  \subsection{The adaptive algorithm}

  The AdaLIPO algorithm  (displayed in Figure \ref{fig:adalipo})
  is an extension of LIPO which involves an estimate
  of the Lipschitz constant and takes as input a parameter $p \in (0,1)$ 
  and a nondecreasing sequence of Lipschitz constant $k_{i \in \mathbb{Z}}$ defining a meshgrid of $\R^+$
  ({\it i.e.~}such that
  $
  \label{eq:lip_constant}
   \forall x >0,~ \exists i \in \mathbb{Z} \text{~~~with~~} k_i \leq x \leq k_{i+1}
  $.)
  The algorithm is initialized with a Lipschitz constant $\hat{k}_1$ set to $0$ and 
  alternates randomly between two distinct phases: 
  {exploration} and {exploitation}.
  Indeed,  at step $t <n$, a  Bernoulli random variable 
  $B_{t+1}$ of parameter $p$  which drives this trade-off is sampled.
  If $B_{t+1}=1$, then the algorithm explores the space 
  by evaluating the function over a point uniformly sampled over $\X$.
  Otherwise, if $B_{t+1}=0$, the algorithm exploits the previous evaluations
  by making an iteration of the LIPO algorithm
  with the smallest Lipschitz constant of the sequence $\hat{k}_t$
  which is associated with a subset of Lipschitz functions
  that probably contains $f$ (step abbreviated in the algorithm
  by $X_{t+1}\sim \mathcal{U}(\X_{\hat{k}_t,t})$).
  Once an evaluation has been made,  the Lipschitz constant estimate $\hat{k}_t$
  is updated.

\RestyleAlgo{boxed}
\begin{figure}[t!]
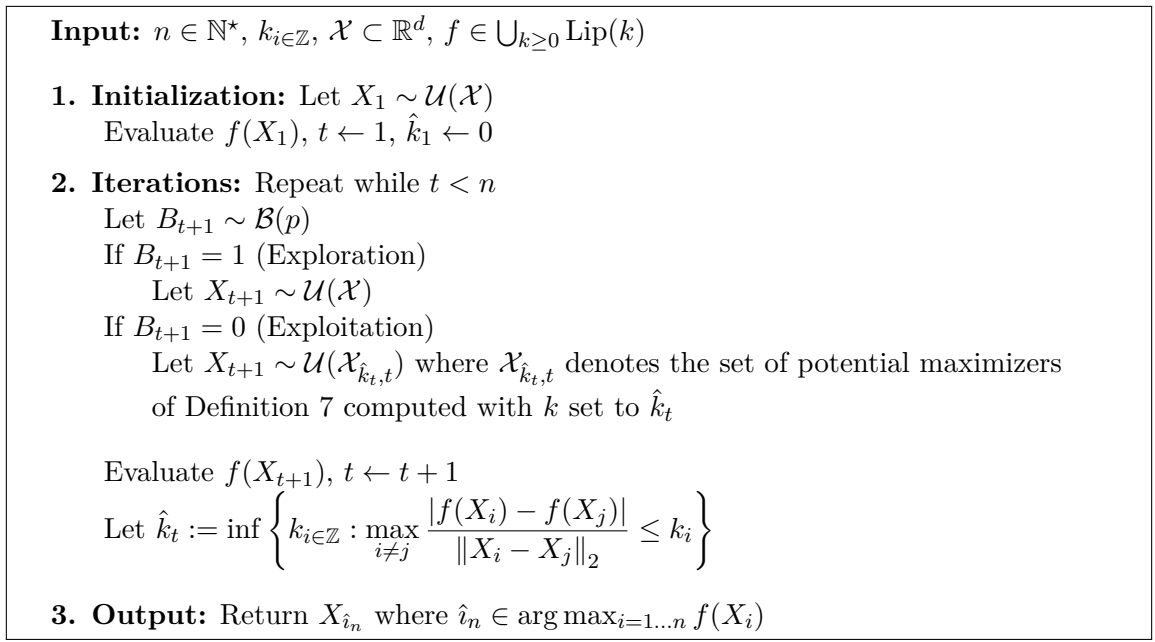

  \begin{algorithm}[H]
  \textbf{Input:} $n \in \mathbb{N}^{\star}$, $k_{i \in \mathbb{Z}}$,
  $\X\subset \R^d$, $f \in \bigcup_{k\geq 0}\Lip(k)$\\%
  \vspace{1em}
 \textbf{1. Initialization:} Let $X_1 \sim \mathcal{U}(\X)$\\
 \textcolor{white}{.....} Evaluate $f(X_1)$, 
 $t \leftarrow 1$, $\hat{k}_1 \leftarrow 0$\\\vspace{-0.7em}
 ~\\
 \textbf{2. Iterations:} Repeat while $t <n$ \\
	\textcolor{white}{.....} Let $B_{t+1} \sim \mathcal{B}(p)$\\
    \textcolor{white}{.....} If  $B_{t+1} = 1$ (Exploration)\\
    \textcolor{white}{...........} Let $X_{t+1}\sim\mathcal{U}(\X)$ \\
    \textcolor{white}{.....} If  $B_{t+1} = 0$ (Exploitation)\\
    \textcolor{white}{...........} Let $X_{t+1}\sim\mathcal{U}(\X_{\hat{k}_t,t})$ 
    where $\X_{\hat{k}_t,t}$ denotes the set of  potential maximizers\\
    \textcolor{white}{...........}  of Definition \ref{def:potential}
    computed with $k$ set to $\hat{k}_t$\\\vspace{1em}
    \textcolor{white}{.....} Evaluate $f(X_{t+1})$, $t \leftarrow t+1$\\
    \textcolor{white}{.....} Let 
    $\displaystyle{\hat{k}_t := \inf\left\{ k_{i \in \mathbb{Z}}
    : \max_{i\neq j} \frac{|f(X_i) -f(X_j)|}{\norm{X_i-X_j}_2}  \leq k_i \right\} }$\\\vspace{1em}
\textbf{3. Output:} Return $X_{\hatin}$ where $\hatin \in \arg\max_{i=1\dots n}f(X_i)$
\end{algorithm}
  \vspace{-0.5em}
  \caption{The AdaLIPO algorithm}
  \label{fig:adalipo}
\end{figure}

    \begin{remark}
   {\sc (Examples of meshgrids)}
   Several sequences of Lipschitz
   constants with various shapes could be considered such as 
   $k_i=$ $|i|^{\sgn{i}}$, $\ln(1 + |i|^{\sgn{i}})$ or $(1+\alpha)^i$ for some $\alpha>0$.
   It should be noticed in particular
   that the computation of the estimate is straightforward with these sequences.
   For instance, when $k_i = (1+\alpha)^i$, we have 
   $\hat{k}_t = (1+\alpha)^{i_t}$ where  $i_t = \ceil{ \ln(\max_{i\neq j}|f(X_j)
  -f(X_l)|/\lVert X_j -X_l \rVert_2)/ \ln(1+\alpha)}$.
  \end{remark}

  \begin{remark}
   {\sc (Alternative Lipschitz constant estimate)} 
   Due to the generecity of the algorithm,
   we point out that any Lipschitz constant estimate such as the one proposed 
   in \cite{wood1996estimation} or \cite{bubeck2011lipschitz}
   could also be considered to implement the algorithm.
   However, as the analysis requires the estimate to be universally consistent
   (see Section below),
   we will only consider the proposed one that presents such a property.
  \end{remark}

  \subsection{Convergence analysis}

  \noindent{\bf Lipschitz constant estimate.} 
  Before starting the analysis of the algorithm,
  we first provide a control on the  Lipschitz constant 
  estimate based on a sample of random evaluations that will be useful to analyse 
  its performance.
  More precisely, the next result illustrates the purpose 
  of using  a discretization of Lipschitz constant instead
  of a raw estimate of the maximum slope 
  by showing that, given this estimate,
  a small subset of functions containing the unknown function can be recovered in a finite-time.

  \begin{proposition}
  \label{prop:lip_estimate}
  Let $f\in\bigcup_{k \geq 0}\Lip(k)$ be any non-constant Lipschitz function.
  Then, if $\hat{k}_t$ denotes the Lipschitz constant estimate
  of Algorithm 2 computed with any increasing sequence $k_{i \in\mathbb{Z}}$ defining a meshgrid of $\R^+$
  over a sample $(X_1, f(X_1)),\dots, (X_t, f(X_t))$ of $t\geq 2$
  evaluations where $X_1,\dots, X_t$ are uniformly and independently distributed over $\X$,
  we have that
  \[
  \P\left( f \in \normalfont{\text{Lip}}(\hat{k}_t) \right) 
  \geq 1 - (1 - \Gamma(f, k_{i^{\star}\minus1}) )^{\lfloor t/2 \rfloor }
  \]
  where the coefficient
  \[
   \Gamma(f, k_{i^{\star}\minus1}) := 
   \P\left( \frac{\abs{f(X_1) - f(X_2)}}{\norm{X_1-X_2}_2} > k_{i^{\star}\minus1}\right)>0
  \]
  with $i^{\star} = \min\{ i\in\mathbb{Z} : f \in \Lip(k_i) \}$,
  is strictly positive.
  \end{proposition}
  The following remarks provide some insights on the quantities
  involved in the bound.

  \begin{remark}
   {\sc (Measure of global smoothness)} 
   The coefficient $\Gamma(f,k_{i^{\star}\minus1})$ which appears in the lower bound 
   of Proposition \ref{prop:lip_estimate} can 
   be seen as a measure of the global smoothness
   of the function $f$ with regards to $k_{i^{\star}\minus1}$.
   Indeed, observing that
   $(1/\floor{t/2})\cdot\sum_{i=1}^{\floor{t/2}}$ $
    \mathbb{I}\{|f(X_i) - f(X_{i+ \floor{t/2}}) | > 
    k_{i^{\star}\minus1} \lVert X_i - X_{\floor{t/2}+i}\rVert_2  \}$ $\xrightarrow{p}
    \Gamma(f,k_{i^{\star}\minus1}\!),
   $
   it is easy to see that this coefficient records the ratio of volume
  the product space $\X\times\X$ where $f$ 
  is witnessed to be at least $k_{i^{\star}\minus1}$-Lipschitz.
  \end{remark}

    \begin{remark}
   {\sc (Density of the sequence of Lipschitz constants)} 
   As a 
   consequence of the previous remark,
   we point out that the density of the sequence of Lipschitz constants
   $k_{i\in\mathbb{Z}}$,
   captured here by $\alpha = \sup_{i \in \mathbb{Z}}(k_{i+1}-k_i)/k_i$,
   has opposite impacts on the maximal deviation of the estimate
   and its convergence rate. Indeed, as
   $\alpha$ is involved in both the following upper bounds on
   the deviation and on the coefficient Gamma: 
   $
    (\lim_{t \to \infty} \hat{k}_t-k^{\star})/k^{\star} \leq \alpha
    $
    and
    $
    \Gamma(f,k_{i^{\star}\minus1}\!)\leq \Gamma(f,k^{\star}/(1+\alpha))
    $
   where $k^{\star} = \sup\{k\geq 0: f \notin \Lip(k) \}$ denotes the optimal 
   Lipschitz constant,
   we deduce that using a dense sequence of Lipschitz constants
   with a small $\alpha$ reduces the bias
   but also the convergence rate through a small coefficient $\Gamma(f, k_{i^{\star}\minus 1})$.
  \end{remark}

     \begin{remark}
     {\sc (Impact of the dimensionality)} 
     Last, we provide a simple result which illustrates the fact that,
     independently of the function,
     the task of estimating the Lipschitz constant becomes harder
     as the dimensionality $d$ grows large.
    Let $f : [0,1]^d \to \R$  be any k-Lipschitz function 
    for some $k\geq 0$ with regards to the Euclidean distance.
    Then, for all $t >0 $,
    \[
     \P\left\{ \frac{ \abs{f(X) -f(X') } }{~\norm{X-X'}_2} > tk\right\} \leq 5	 e^{ - cdt^2 }
    \]
    where $X$ and $X'$ are uniformly distributed over $[0,1]^d$ and
    $c>0$ is some absolute constant.
    Combined with the naive bound 
    $\P(f \in \Lip(\hat{k}_t)) \leq t^2\Gamma(f,k_{i^{\star}-1})$,
    this result shows that one should collect at least 
    $\Omega_{\P}(e^{c d k/k_{i^{\star} \minus 1 } })$
    evaluation points to succefully estimate the constant  $k_{i^{\star}}$.
   \end{remark}
%
  Equipped with this result, we may now turn to the analysis of AdaLIPO.\\
  
  \noindent {\bf Analysis of AdaLIPO.} 
  Given the consistency equivalence of Proposition 
  \ref{prop:consistency_equivalence},
  one can directly obtain the following asymptotic result.
  \begin{proposition}
  \label{prop:consistency_adalipopt}
  {\sc (Consistency)}
    The AdaLIPO algorithm tuned with
    any parameter $p\in (0,1)$
    and any sequence of Lipschitz constant $ k_{i \in \mathbb{Z}}$ 
    defining a meshgrid of $\R^{+}$\! is
    consistent over the set of Lipschitz functions, {\it i.e.,}
    \[
      \forall f \in \textstyle{ \bigcup_{k \geq 0}\Lip(k)}
      ,~~\displaystyle{\max_{i=1\dots n}f(X_i) \xrightarrow{p}  \max_{x \in \X} f(x)}.
    \]
  \end{proposition}
  The next result provides a first finite-time bound on the 
  difference between the maximum and its approximation.
  \begin{proposition}
  \label{prop:adagenericrate}
   {\sc (Upper Bound)} Consider AdaLIPO
   tuned with any $p\in(0,1)$ and
   any sequence $k_{i\in\mathbb{Z}}$ defining a meshgrid of $\R^{+}$\!.
   Then, for any non-constant $f \in\bigcup_{k\geq 0}\Lip(k)$, 
   any $n\in\mathbb{N}^{\star}$ and
   $\delta \in (0,1)$, we have 
   with probability at least $1-\delta$,
   \begin{align*}
    \max_{x \in \X}f(x) - \max_{i=1\dots n}f(X_i) \leq
    k_{i^{\star}} \times \diam{\X}
    \times \left( \frac{5}{p}+  \frac{2\ln(\delta/3)}{p\ln(1-\Gamma(f,k_{i^{\star}
    \text{-}1}))} \right)^{\frac{1}{d}}
    \times \left(  \frac{\ln(3/\delta)}{n} \right)^{\frac{1}{d}}\!\!
   \end{align*}
   where $\Gamma(f,k_{i^{\star}\minus1})$ and $i^{\star}$ 
   are defined as in Proposition \ref{prop:lip_estimate} and $\ln(0)=-\infty$ 
   by convention.
  \end{proposition}
  This result might be misleading since it advocates
  that doing pure exploration gives the best rate ({\it i.e.}, when $p \to 1$).
  As Proposition \ref{prop:lip_estimate} provides us with 
  the guarantee that $f \in \Lip(\hat{k}_t)$ within a finite number of iterations
  where $\hat{k}_t$ denotes the Lipschitz constant estimate,
  one can however recover faster convergence rates similar to 
  the one reported for LIPO where the constant $k$ is assumed to be known.

    \begin{theorem}
    \label{th:fastada}
   {\sc (Fast Rates)} 
   Consider the same assumptions as in Proposition \ref{prop:adagenericrate}
   and assume in addition that the function $f$ satisfies Condition \ref{cond}
   for some $\kappa\geq 1$, $c_{\kappa}\geq 0$.
   Then, for any $n\in\mathbb{N}^{\star}$ and $\delta \in (0,1)$, 
   we have with probability at least $1-\delta$,
   \begin{flushleft}
   $\displaystyle{\max_{x \in \X}f(x) -\max_{i=1\dots n} f(X_i) 
   \leq k_{i^{\star}} \times \diam{\X}}\times$
   \end{flushleft}
  \begin{flushright}
   $
   \exp\left( \frac{2\ln(\delta/4)}{p\ln(1-\Gamma(f,k_{i^{\star}\minus 1}))} 
    + \frac{7\ln(4/\delta)}{p(1-p)^2}
    \right) \times 
   \begin{dcases*}
     \exp\left\{- C_{k_{i^{\star}}, \kappa}\cdot
	\frac{n ~ (1-p)\ln(2)}{2\ln(n/\delta) + 4(2\sqrt{d})^d }
	\right\}   \,  ,~~~~~~~~~~\kappa =1  \\
	~\\
      2^{\kappa}
	\left( 
	1 + C_{k_{i^{\star}}, \kappa}\cdot \frac{n(1-p)  (2^{d(\kappa-1)}-1)}{2\ln(n/\delta) +4(2\sqrt{d})^d}
	\right)^{-\frac{\kappa}{d(\kappa-1)}}\!\!\!\!\!\!\!\!\!\!\!\!\!\!\!\!\!\!\!\!\!\!
~,~~~~~~~~~ \kappa>1
    \end{dcases*}
   $
  \end{flushright}
  where $C_{k_{i^{\star}} \kappa} = 
  (c_{\kappa} \max_{x \in \X} \norm{x-x^{\star}}_2^{\kappa-1} /8 k_{i^{\star}})^d$.
  \end{theorem}
  This bound shows the precise impact of the parameters $p$ and $k_{i \in \mathbb{Z}}$
  on the convergence of the algorithm. 
  It illustrates the complexity of the exploration/exploitation
  trade-off through a constant term and a convergence rate
  which are inversely correlated to the exploration parameter
  and the density of the sequence of Lipschitz constants.
  Recall also that as these bounds are of the same order as when $k$ is known,
  the examples of Section \ref{sec:examples} still remain valid.
  We may now compare our results with existing works.
  
  \subsection{Comparison with previous works}
  \label{sec:comparison_ada}
  
  {\bf Algorithms.} The DIRECT algorithm (\cite{jones1993lipschitzian}) is a 
  Lipschitz optimization algorithm where the Lipschitz constant is unknown.
  It uses a deterministic splitting technique of
  the search space in order to sequentially divide and evaluate the function 
  over a subdivision of the space that have recorded the highest 
  upper bound among all subdivisions of similar size
  for at least a possible value of $k$.
  \cite{munosmono} generalized DIRECT in a broader setting by 
  extending the 
  the DOO algorithm to any
  unknown and arbitrary local semi-metric under the name SOO.
  With regards to these works, 
  we proposed an alternative stochastic strategy which directly relies 
  on the estimation of the Lipschitz constant and thus only presents guarantees
  for globally Lipschitz functions.\\
  
  \noindent {\bf Results.} Up to our knowledge, only the consistency property of DIRECT
  was shown in \cite{finkel2004convergence} and \cite{munosmono}
  derived convergence rates for SOO using weaker local smoothness assumptions.
  To compare our results,  we considered SOO tuned with the depth 
  function $h_{\max}(t) = \sqrt{t}$ providing the best rate,
  over the domain $\X=[0,1]^d$ partitioned into a $2^d$-ary tree 
  of hypercubes and with $f$ belonging to the following sets of globally Lipschitz 
  functions: (a) $\bigcup_{k\geq 0}\Lip(k)$,
  (b) $\mathcal{F}_{\kappa}= \{ f \in \bigcup_{k\geq 0}\Lip(k)$ satisfying Condition \ref{cond}
  with $c_{\kappa}, \kappa \geq 1\}$
  and (c) $\mathcal{F}'_{\kappa} = \{ f \in \mathcal{F}_{\kappa}:
  f(x^{\star}) -f(x) \leq c_2 \norm{x-x^{\star}}^{\kappa}_2\}$.
  The result of the comparison can be found in Table \ref{tab:comparisonlipopt}.
  In addition to the novel rate over the class of Lipschitz functions,
  we were also able to obtain a faster polynomial rate than 
  SOO over the set $\F_{\kappa}$.
  However, SOO achieves its best rate of order
  $O(e^{-c\sqrt{n}})$  over the whole set $\{\F_{\kappa}', \kappa\geq 1\}$
  by adapting similarly to any function locally equivalent to $\norm{x^{\star}-x}_2^{\kappa}$
  while AdaLIPO achieves a slower polynomial rate in the case where $\kappa>1$ 
  but an even faster exponential rate of order $O_{\P}^{\star}(e^{-cn})$ when $\kappa=1$.
  Hence, by exploiting the global smoothness of the function (AdaLIPO),
  we were  able to derive a faster convergence rate than the best one reported for the algorithm 
  exploiting the local smoothness (SOO)
  which however remains valid over a larger subset of functions.

      \begin{table}[!h]
     \vspace{0.5em}
  \centering
  {\small
  \begin{tabular}{@{}lccccc@{}}
  \toprule
  \textbf{Algorithm} 
  & {\textbf{AdaLIPO}} 
  & \textbf{DIRECT} 
  & \textbf{PRS} 
  & \textbf{SOO} \\ \midrule
  $f\in\bigcup_{k\geq 0}\Lip(k)$
  &
  &
  &
  & \\
  Consistency 
  & \checkmark
  & \checkmark
  & \checkmark
  & \checkmark       \\[-.5ex]
  Upper Bound 
  & $O^{*}_{\P}(n^{-\frac{1}{d}})$ 
  & - 
  & $O_{\P}(n^{-\frac{1}{d}})$ & - \\
  &
  &
  &
  &
  \\[-3ex]
  \\\cmidrule[0.5\lightrulewidth]{2-5}
  $f \in \F_{\kappa}$, $\kappa>1$
  &
  &
  &
  &
  \\[-0.5ex]
  Upper bound 
  & $O^{*}_{\P}(n^{-\frac{ \kappa}{d(\kappa~\!\minus~\!1)}})$
  & -
  & $O_{\P}(n^{-\frac{1}{d}})$
  & $O(n^{-\frac{ \kappa}{2d(\kappa~\!\minus~\!1)}})$\\
  Lower bound
  & -
  & -
  & $\Omega_{\P}(n^{-\frac{\kappa}{d}})$
  & - \\
  &
  &
  &
  &
  \\[-3ex]
  \\\cmidrule[0.5\lightrulewidth]{2-5}
  $f \in \F'_{\kappa}$, $\kappa>1$  
  &
  &
  &
  & \\[-0.5ex]
  Upper bound 
  & $O^{*}_{\P}(n^{-\frac{ \kappa \times \kappa}{d(\kappa~\!\minus~\!1 )}})$ 
  & -
  & $O_{\P}(n^{-\frac{\kappa}{d}})$ 
  &  $O(e^{-\frac{\sqrt{n}\ln(2) }{ (\sqrt{2d}(c_2/c_{\kappa})^{1/\kappa})^d } })$ \\ 
  Lower bound 
  & -
  & -
  & $\Omega_{\P}(n^{-\frac{\kappa}{d}})$
  & - \\
  \\[-3ex]
  \\\cmidrule[0.5\lightrulewidth]{2-5}
  $f \in \F'_{\kappa}$, $\kappa=1$ 
  &
  &
  &
  & \\[-0.5ex]
  Upper bound
  & $O_{\P}^{*}(e^{ -\frac{n(1-p) \ln(2)}{4(16k_{i^{\star}}\sqrt{d}/c_{\kappa})^d} })$ 
  & -
  & $O_{\P}(n^{-\frac{1}{d}})$
  & $O(e^{-\frac{\sqrt{n}\ln(2) }{ (c_2\sqrt{2d}/c_{\kappa})^d } })$ \\ 
  Lower bound 
  & -   
  & -
  & $\Omega_{\P}(n^{- \frac{1}{d}})$  
  & -\\ 
  \bottomrule
  \end{tabular}
  \vspace{-0.2em}
  \caption{Comparison of the results reported over the difference
  $\max_{x \in \X}f(x) - \max_{i=1\dots n}f(X_i)$ in global optimization literature.
  Dash symbols are used when no results could be found.}
  \label{tab:comparisonlipopt}}
  \end{table}

  \section{Numerical experiments} 
\label{sec:experiments}

  In this section, we compare the empirical performance of 
  AdaLIPO to existing state-of-the-art global optimization methods
  on real and synthetic optimization problems.\\

  \noindent {\bf Algorithms.}\footnote{In Python 2.7 from the libraries: $^{*}$BayesOpt \cite{martinez2014bayesopt}, 
$^{\ddag}$CMA 1.1.06 \cite{CMAES_implementation} and
$^{\dag}$NLOpt \cite{johnson2014nlopt}.}
  Six different types of algorithms developed from various approaches of global optimization
  were considered in addition to AdaLIPO:
  \begin{itemize}
    \item[-] {\textbf{BayesOpt}}$^{*}$ (\cite{martinez2014bayesopt})
    is a Bayesian optimization algorithm.
    It uses a distribution over functions to build a surrogate model of the unknown function.
    The parameters of the distribution are estimated during the optimization process.
    \item[-] \textbf{CMA-ES}$^{\ddag}$ (\cite{hansen2006cma})
    is an evolutionary algorithm. It samples
    the next evaluation points according to a multivariate
    normal distribution with mean vector and covariance matrix
    computed from the previous evaluations.
    \item[-] \textbf{CRS}$^{\dagger}$ (\cite{kaelo2006some})
    is a variant of {\sc PRS} which includes local mutations. It
    starts with a random population and evolves these points by an heuristic rule.
    \item[-] \textbf{\textsc{DIRECT}} (\cite{jones1993lipschitzian}) 
    is the Lipschitz optimization
    algorithm with unknown Lipschitz introduced in Section \ref{sec:comparison_ada}.
    \item[-] \textbf{MLSL}$^{\dagger}$ (\cite{kan1987stochastic})
    is a multistart algorithm. It
    performs a series of local optimizations starting from points randomly chosen 
    by a clustering heuristic that helps to avoid repeated searches of the same local optima.
    \item[-] \textbf{PRS}
    is the standard random covering method described in Section \ref{sec:global_opt}
    (see Example \ref{ex:PRS}).
  \end{itemize}
    For a fair comparison, the tuning parameters were all set to default 
    and AdaLIPO was constantly used
    with a parameter $p$ set to $0.1$ and a sequence $k_i= (1+0.01/d)^i$
    fixed by an arbitrary rule of thumb.

~\\
  \noindent {\bf Data sets.} Following the steps of (\cite{malherbe2016ranking}),
  we considered a series of global optimization problems involving reals data sets
  and a series of synthetic problems:
  \begin{itemize}
   \item[I)]    First, we studied the task of estimating the regularization parameter $\lambda$
    and the bandwidth $\sigma$ of a gaussian kernel ridge regression
    minimizing the empirical mean squared error of the predictions
    over a 10-fold cross validation with real data sets. 
    The optimization was performed over $(\ln(\lambda), \ln(\sigma)) \in [-3,5] \times [-2,2]$
    with five data sets from the UCI Machine Learning Repository \cite{Lichman:2013}: 
    {\it Auto-MPG}, {\it Breast Cancer Wisconsin (Prognostic)},
    {\it Concrete slump test}, {\it Housing} and {\it Yacht Hydrodynamics}.
    \item[II)]  Second, we compared the algorithms on a series of five synthetic problems 
    commonly met in standard optimization benchmark
    taken from (\cite{jamil2013literature, simulationlib}):
    {\it HolderTable}, {\it Rosenbrock}, {\it Sphere}, {\it LinearSlope} and {\it Deb N.1}.
    This series includes multimodal  
    and non-linear functions as well as ill-conditioned and well-shaped functions
    with a dimensionality ranging from 2 to 5.
  \end{itemize}
    A complete description of the test functions of the benchmark
    can be found in Table \ref{tab:func}.

\begin{table}[!h]
\begin{center}
{ \footnotesize \tt
\begin{tabular}{@{}llcc@{}}
\toprule
 {\bf Problem} & {\bf Objective function}  & {\bf Domain}  & {\bf Local max.} \\ \midrule
  & & & \\
   {\bf Auto MPG} 
 &  
 &[-2,4]$\times$[-5,5]
 & - \\
 &\multirow{-2}{*}{-$\displaystyle \tt \frac{1}{10}\sum_{k=1}^{10} \sum_{i \in D_k}$($\tt \hat{f}_k(X_i) -Y_i $)$^{\tt 2}$  } 
 &  &  \\
  &  &  &  \\
 {\bf Breast Cancer~~} 
 & \multicolumn{1}{l}{where:} 
 &~\![-2,4]$\times$[-5,5]
 & - \\
 & \multicolumn{1}{l}{- $\tt \hat{f}_k \in$ $\tt\underset{f\in \mathcal{H}_{\sigma}}{argmin}$
 $\tt \frac{1}{n- |D_k|} \sum_{i \notin \mathcal{D}_k } (f(X_i)- Y_i)^2
 + \lambda \norm{f}_{\mathcal{H}_{\sigma}}$}
 &  &  \\
 &  
 \multicolumn{1}{l}{- the data set \{($\tt X_i, Y_i$)\}$_{\tt i=1}^n$ is split}
 &  &  \\
 {\bf Concrete} 
 & \multicolumn{1}{l}{~~into 10 folds $\tt D_1\dots D_{10}$}
 & [-2,4]$\times$[-5,5]
 & - \\
 & \multicolumn{1}{l}{- $\mathcal{H}_{\sigma}$ denotes the gaussian RKHS of}
 &
 &  \\
 &  
 \multicolumn{1}{l}{~~bandwidth $\sigma$}
 &  &  \\
  {\bf Yacht} 
 & \multicolumn{1}{l}{- $\tt \norm{f}_{\mathcal{H}_{\sigma}}$ is the corresponding norm}
 & [-2,4]$\times$[-5,5]
 &  - \\
 &
 &
 &  \\
  & 
  \multicolumn{1}{l}{- $\tt\sigma =10^{x_1}$}
  &  &  \\
 {\bf Housing} 
 & 
 & [-2,4]$\times$[-5,5]
 & -
 \\
 & \multicolumn{1}{l}{- $\tt\lambda = 10^{x_2}$}
 &  &  \\
 \cmidrule[0.2pt]{1-4}\\
  {\bf Holder Table} 
 & |sin($\tt x_1$)|$\times$|cos($\tt x_2$)| 
 & [-10,10]$^{\tt 2}$
 &  36 \\
 & $ \times$exp(|1-($\tt x_1^2\text{~\!+~\!}x_2^2\text{)}^{1/2}\text{/}\pi$|) 
 &
 &  \\
  &  &  &  \\
 {\bf Rosenbrock} 
&  -~\!$\tt \sum_{i=1}^2 [100(x_{i+1}-x_i^2)^2+(x_i-1)^2]$
 &[-2.048,2.048]$\tt^3$
 & - \\
 &
 &  &  \\
  &  &  &  \\
 {\bf Sphere} & -$\tt \left(\sum_{i=1}^4(x_i- \pi/16)^2 \right)^{1/2}$
 & [0,1]$\tt^4$
 & 1 \\
 & 
 &
 &  \\
 &  &  &  \\
  {\bf Linear Slope }
 &  $\tt\sum_{i=1}^4 10^{\text{(}i-1\text{)/}4 } \text{(}x_i$-~\!\!5)
 &  [-5,5]$^{\tt4}$
 & 1 \\
 &  &  &  \\
 &  &  &  \\
  {\bf Deb N.1} 
 &  $\tt \frac{1}{5} \sum_{i=1}^5$sin$^{\tt6}$(5$\pi x_i$)
 & [-5,5]$^{\tt5}$
 & 36 \\
 &  &  &  \\
 \bottomrule
\end{tabular}
\vspace{-2.5em}
}
\end{center}
\caption{Description of the test functions of the benchmark.
Dash symbols are used when a value could not be calculated.}
\label{tab:func}
\end{table}

~\\
    \noindent {\bf Protocol and performance metrics.} For each problem and each algorithm,
    we performed $K\!=\!\!100$ distinct runs with a budget of $n\!=\!\!1000$ 
    function evaluations.
    For each target parameter $t=$ 90\%, 95\% and 99\%, we have collected
    the stopping times corresponding to the number of evaluations 
    required by each method to reach the specified target 
     $$
      \tau_{k} := \min\{i=1,\dots, n:~
      f(X^{(k)}_{i}) \geq f_{\text{target}}(t)\} 
    $$
    where $\min\{ \emptyset \} = 1000$ by convention, $\{ f(X^{(k)}_{i})\}_{i=1}^n$ denotes 
    the evaluations made by a given method on the $k$-th run with
    $k \leq K$ and the target value is set to 
    $$
      f_{\text{target}}(t) := \max_{x \in \X}f(x)  - 
      \left(\max_{x \in \X}f(x) -\int_{x \in \X} f(x)~\text{d}x/\mu(\X)
      \right) \times (1 - t).
    $$
    The normalization of the target to the average value
    prevents the performance measures from being dependent of any 
    constant term in the unknown function. 
    In practice, the average was estimated from a Monte Carlo sampling 
    of $10^6$ evaluations and the maximum 
    by taking the best value observed over all the sets of experiments. 
    Based on these stopping times, 
    we computed  the average and standard deviation of the number
    of evaluations required to reach the target, {\it i.e.~}
    $$
    \bar{\tau}_K = \frac{1}{K}\sum_{k=1}^K\tau_k \text{~~and~~} \hat{\sigma}_{\tau} = 
    \sqrt{\frac{1}{K}\sum_{k=1}^K (\tau_k - \bar{\tau}_K)^2}.
    $$

  ~\\
  \noindent {\bf Results.}  Results are collected in Figure \ref{tab:results}. 
  Our main observations are the following.
  First, we point out that the proposed method displays very competitive results 
  over most of the problems of the benchmark
  (exepct on the non-smooth  {\it DebN.1} where most methods fail).
  In particular, AdaLIPO obtains several times the best performance for the target $90$\% and
  $95$\% (see, {\it e.g.}, {\it BreastCancer, HolderTable, Sphere})
  and experiments {\it Linear Slope} and {\it Sphere}
  also suggest that, in the case of smooth functions,
  it can be robust against the
  dimensionality of the input space.
    However, in some cases, the algorithm can be witnessed  to 
  reach  the 95\% target with very few evaluations while getting more slowly to the 99\% target
  (see, {\it e.g.}, {\it Concrete, Housing}).
  This problem is due to the instability of the Lipschitz constant estimate
  around the maxima but could certainly be solved 
  with the addition  of a noise parameter that would
  allow the algorithm be more robust against local perturbations.
  Additionally, investigating better values for $p$ and $k_{i}$
  as well as alternative covering methods such as LHS \cite{stein1987large}
  could also be promising approaches to improve its performance.
  However, an empirical analysis of the algorithm
  with these extensions is beyond the scope of the paper
  and will be carried out in a future work.

  \begin{figure*}[!t]
    \centering
  {\footnotesize
\centering
   \begin{tabular}{lccccccccccp{0.00025em}}
\vspace{-2.1em}
\\\cmidrule[0.12em](r{\dimexpr5.74cm+\tabcolsep-6.2cm\relax}){1-6}
{\bf \!\!\!\!Problem} 
& {\bf Auto-MPG} &
{\bf BreastCancer} 
& {\bf Concrete}
& {\bf Housing} 
& {\bf Yacht}
& \!\!\!\!\!\!\!\! 
\\\cmidrule[\lightrulewidth](r{\dimexpr5.74cm+\tabcolsep-6.2cm\relax}){1-6}
{\bf \!\!\!\!AdaLIPO\!\!\!\!\!\!\!\!} 
& 14.6 ($\pm$09) & {\bf05.4} ($\pm$03)& {\bf04.9} ($\pm$02) 
& {\bf05.4} ($\pm$04) & 25.2 ($\pm$21)\\
{\bf \!\!\!\!BayesOpt} & {\bf 10.8} ($\pm$03) & 06.8 ($\pm$04)& 06.4 ($\pm$03) & 07.5 ($\pm$04) & 13.8 ($\pm$20)
 \\
{\bf \!\!\!\!CMA-ES} & 29.3 ($\pm$25) & 11.1 ($\pm$09)& 10.4 ($\pm$08) & 12.4 ($\pm$12) & 29.6 ($\pm$25)
\\
{\bf \!\!\!\!CRS} & 28.7 ($\pm$14) & 08.9 ($\pm$08) & 10.0 ($\pm$09) & 13.8 ($\pm$10) & 32.6 ($\pm$15)
\\
{\bf \!\!\!\!DIRECT} & 11.0 ($\pm$00) & 06.0 ($\pm$00)& 06.0 ($\pm$00) & 06.0 ($\pm$00) & {\bf11.0} ($\pm$00)
\\
{\bf \!\!\!\!MLSL} & 13.1 ($\pm 15$) & 06.6 ($\pm$03) & 06.1 ($\pm$04) & 07.2 ($\pm$03) & 14.4 ($\pm$13)
\\
{\bf \!\!\!\!PRS} & 65.1 ($\pm$62) & 10.6 ($\pm$10) & 09.8 ($\pm$09) & 11.5 ($\pm$10) & 73.3 ($\pm$72) 
& \hspace{-1.3em} \multirow{-9}{*}{\begin{turn}{270}{\bf~~~~~target 90\%}\end{turn}}\\

 \vspace{-1.2em}
 \\\cmidrule[\lightrulewidth](r{\dimexpr5.74cm+\tabcolsep-6.2cm\relax}){1-6}
{\bf \!\!\!\!AdaLIPO\!\!\!\!\!\!\!\!} & 17.7 ($\pm$09) & {\bf06.6} ($\pm$04)& {\bf06.4} ($\pm$04) & 17.9 ($\pm$25) & 33.3 ($\pm$26)\\
{\bf \!\!\!\!BayesOpt} & 12.2 ($\pm$06) & 08.4 ($\pm$03) & 07.9 ($\pm$03) & {\bf13.9} ($\pm$22)& {\bf15.9} ($\pm$21)\\
{\bf \!\!\!\!CMA-ES} & 42.9 ($\pm$31) & 13.7 ($\pm$10) & 13.5 ($\pm$10) & 23.0 ($\pm$16) & 40.5 ($\pm$30)\\
{\bf \!\!\!\!CRS} & 35.8 ($\pm$13) & 13.6 ($\pm$10) & 14.6 ($\pm$11) & 22.8 ($\pm$12) & 38.3 ($\pm$31)\\
{\bf \!\!\!\!DIRECT} & {\bf 11.0} ($\pm$00) & 11.0 ($\pm$00) & 11.0 ($\pm$00) & 19.0 ($\pm$00) & 27.0 ($\pm$00)\\
{\bf \!\!\!\!MLSL} & 15.0 ($\pm 15$) & 07.6 ($\pm$03) & 07.3 ($\pm$04) & 16.3 ($\pm$10) & 16.3 ($\pm$13)\\
{\bf \!\!\!\!PRS} &  139 ($\pm$131) & 17.7 ($\pm$17) & 14.0 ($\pm$12) & 39.6 ($\pm$39)& 247($\pm$249)
& \hspace{-1.3em} \multirow{-9}{*}{\begin{turn}{270}{\bf ~~~~~target 95\%}\end{turn}}\\

 \vspace{-1.2em}
 \\\cmidrule[\lightrulewidth](r{\dimexpr5.74cm+\tabcolsep-6.2cm\relax}){1-6}
{\bf \!\!\!\!AdaLIPO\!\!\!\!\!\!\!\!} & 32.6 ($\pm$16) & 34.1 ($\pm$36) & 70.8 ($\pm$58) & 65.4 ($\pm$62) & 61.7 ($\pm$39)\\
{\bf \!\!\!\!BayesOpt} & {\bf14.0} ($\pm$07) & 31.0 ($\pm$51) & 28.2 ($\pm$34) & 17.9 ($\pm$22)& {\bf18.5} ($\pm$22)\\
{\bf \!\!\!\!CMA-ES} & 73.7 ($\pm$49) & 35.1 ($\pm$20) & 46.3 ($\pm$29) & 61.5 ($\pm$85) & 70.9 ($\pm$50)\\
{\bf \!\!\!\!CRS} & 48.5 ($\pm$16) & 34.8 ($\pm$12) & 36.6 ($\pm$15) & 43.7 ($\pm$14) & 52.9 ($\pm$18)\\
{\bf \!\!\!\!DIRECT} & 47.0 ($\pm$00) & 27.0 ($\pm$00) & 37.0 ($\pm$00) & 41.0 ($\pm$00) & 49.0 ($\pm$00)\\
{\bf \!\!\!\!MLSL} & 20.6 ($\pm$17) & {\bf12.8} ($\pm$03) & {\bf14.7} ($\pm$10) & {\bf16.3} ($\pm$10) & 21.4 ($\pm$14)\\
{\bf \!\!\!\!PRS} &  747($\pm$330) & 145($\pm$124) & 176($\pm$148) & 406~\!($\pm$312) & 779($\pm$334)
& \hspace{-1.3em} \multirow{-9}{*}{\begin{turn}{270}{\bf ~~~~target 99\%}\end{turn}}\\
\vspace{-1.2em}
\\\cmidrule[0.12em](r{\dimexpr5.74cm+\tabcolsep-6.2cm\relax}){1-6}
\end{tabular}
}
\vspace{1em}

{\footnotesize
\centering
   \begin{tabular}{lcccccp{0.00025em}}
\vspace{-2.1em}
\\\cmidrule[0.12em](r{\dimexpr5.74cm+\tabcolsep-6.2cm\relax}){1-6}
{\bf \!\!\!\!Problem} 
& {\bf HolderTable} & {\bf Rosenbrock}
& {\bf LinearSlope} 
& {\bf Sphere} 
& {\bf Deb N.1}
& \!\!\!\!\!\!\!\! 
\\\cmidrule[\lightrulewidth](r{\dimexpr5.74cm+\tabcolsep-6.2cm\relax}){1-6}
{\bf \!\!\!\!AdaLIPO\!\!\!\!\!\!\!\!} 
& {\bf077} ($\pm$058) & 07.5 ($\pm$07) & 029 ($\pm$13)  & 036 ($\pm$12) & 916($\pm$225)\\
{\bf \!\!\!\!BayesOpt}
& 410 ($\pm$417)  & 07.6 ($\pm$05) & 032 ($\pm$58) & 019 ($\pm$03) & 814($\pm$276) \\
{\bf \!\!\!\!CMA-ES} 
& 080 ($\pm$115) & 10.0 ($\pm$10) & 100 ($\pm$76) & 171 ($\pm$68) & 930($\pm$166) \\
{\bf \!\!\!\!CRS}
& 307 ($\pm$422) & 09.0 ($\pm$09) & 094 ($\pm$43) & 233 ($\pm$54) &  980($\pm$166) \\
{\bf \!\!\!\!DIRECT} 
& 080 ($\pm$000) & 10.0 ($\pm$00) & 092 ($\pm$00) & {\bf031} ($\pm$00) & 1000($\pm$00)\\
{\bf \!\!\!\!MLSL}
& 305 ($\pm$379) & {\bf06.9} ($\pm$05) & {\bf016} ($\pm$33) & 175($\pm$302) & {\bf198}($\pm$326) \\
{\bf \!\!\!\!PRS} 
& 210 ($\pm$202) & 09.0 ($\pm$09) & 831($\pm$283) & 924($\pm$210) & 977($\pm$117)
& \hspace{-1.3em} \multirow{-9}{*}{\begin{turn}{270}{\bf~~~~~target 90\%}\end{turn}}\\

 \vspace{-1.2em}
 \\\cmidrule[\lightrulewidth](r{\dimexpr5.74cm+\tabcolsep-6.2cm\relax}){1-6}
{\bf \!\!\!\!AdaLIPO\!\!\!\!\!\!\!\!} 
& 102 ($\pm$065) & 11.5 ($\pm$11) & 053 ($\pm$22)  & {\bf042} ($\pm$11) & 986($\pm$255)\\
{\bf \!\!\!\!BayesOpt}
& 418 ($\pm$410) & 12.0 ($\pm$08) & 032 ($\pm$59) & 045 ($\pm$16) &  949($\pm$153)\\
{\bf \!\!\!\!CMA-ES} 
& 136 ($\pm$184) & 16.1 ($\pm$13) & 151 ($\pm$94) & 223 ($\pm$57) & 952($\pm$127)\\
{\bf \!\!\!\!CRS} 
& 580 ($\pm$444) & 15.8 ($\pm$14) & 131 ($\pm$62) & 340 ($\pm$66) & 997($\pm$127)\\
{\bf \!\!\!\!DIRECT} 
& {\bf080} ($\pm$000) & 10.0 ($\pm$00) & 116 ($\pm$00) & 098 ($\pm$00) & 1000($\pm$00)\\
{\bf \!\!\!\!MLSL} 
& 316 ($\pm$384) & {\bf08.8} ($\pm$05) & {\bf018} ($\pm$37) & 226($\pm$336) & {\bf215}($\pm$328)\\
{\bf \!\!\!\!PRS} 
& 349 ($\pm$290) & 18.0 ($\pm$17) & 985($\pm$104) & 1000 \!($\pm$00) & 998($\pm$025)
& \hspace{-1.3em} \multirow{-9}{*}{\begin{turn}{270}{\bf ~~~~~target 95\%}\end{turn}}\\

 \vspace{-1.2em}
 \\\cmidrule[\lightrulewidth](r{\dimexpr5.74cm+\tabcolsep-6.2cm\relax}){1-6}
{\bf \!\!\!\!AdaLIPO\!\!\!\!\!\!\!\!} 
& 212 ($\pm$129) & 44.6 ($\pm$39) & 122 ($\pm$31)  & {\bf052} ($\pm$10) & 1000 \!($\pm$00)\\
{\bf \!\!\!\!BayesOpt} 
& 422 ($\pm$407) & 27.6 ($\pm$22) & 032 ($\pm$59) & 222 ($\pm$77) & 1000 \!($\pm$00)\\
{\bf \!\!\!\!CMA-ES} 
& 215 ($\pm$198) & 43.5 ($\pm$37) & 211 ($\pm$92) & 308 ($\pm$60) & 962($\pm$106)\\
{\bf \!\!\!\!CRS}
& 599 ($\pm$427) & 42.7 ($\pm$23) & 168 ($\pm$76) & 607 ($\pm$81) & 1000($\pm$00)\\
{\bf \!\!\!\!DIRECT} 
& {\bf080} ($\pm$000) & 24.0 ($\pm$00) & 226 ($\pm$00) & 548 ($\pm$00) & 1000($\pm$00)\\
{\bf \!\!\!\!MLSL} 
& 322 ($\pm$382) & {\bf19.4} ($\pm$49) & {\bf022} ($\pm$42) & 304($\pm$357) & {\bf256}($\pm$334) \\
{\bf \!\!\!\!PRS} 
& 772 ($\pm$310) & 100\! ($\pm$106) & 1000($\pm$00) & 1000($\pm$00) & 1000($\pm$00)
& \hspace{-1.3em} \multirow{-9}{*}{\begin{turn}{270}{\bf ~~~~target 99\%}\end{turn}}\\
\vspace{-1.2em}
\\\cmidrule[0.12em](r{\dimexpr5.74cm+\tabcolsep-6.2cm\relax}){1-6}
\end{tabular}
}

\caption{
Results of the numerical experiments.
The tables display the 
number of evaluations required by each method to reach the specified target (mean $\pm$ standard 
deviation). In bold, the best result obtained in terms of average of 
function evaluations.
}\label{tab:results}
\vspace{-1.5em}
\end{figure*}

%
  \section{Conclusion}
  
  We introduced two novel strategies for global optimization: 
  LIPO which requires the knowledge of the Lipschitz constant 
  and its adaptive version AdaLIPO which estimates the constant during the optimization process. 
  A theoretical analysis is provided and  empirical results based on synthetic and real problems have
  also been obtained demonstrating the performance of the adaptive  algorithm
  with regards to existing state-of-the-art global optimization methods.

\appendix

\section{Preliminary results}

  We provide here two geometric results 
  (Corollary \ref{coro:covering_number} and Lemma \ref{lem:ball_zab}) 
  and a stochastic result
  (Proposition \ref{prop:cvg_prs})
  that are used repeatidly in the computations.
  We start with the definition of covering numbers.\smallskip

  \begin{definition}
  {\sc (Covering number and $\epsilon$-cover)}
  For any compact and convex set $\X\subset \R^d$
  and any $\epsilon>0$, we say that a sequence 
  $x_1, \dots, x_n$ of $n$ points in $\X$
  defines an $\epsilon$-cover of $\X$ if and only if 
  $
  \X \subseteq\bigcup_{i=1}^n B(x_i, \epsilon).
  $
  The covering number $\mathcal{N}_{\epsilon}(\X)$ of $\X$ is then defined
  as the minimal size of a sequence defining an $\epsilon$-cover of $\X$, 
  i.e.
  \[
   \mathcal{N}_{\epsilon}(\X) :=
   \inf\left\{  n \in \mathbb{N}^{\star} : 
   \exists (x_1,\dots, x_n)\in \X^n \text{~s.t.~} \X
   \subseteq \bigcup_{i=1}^{n} B(x_i, \epsilon) 
   \right\}.
  \]

  \end{definition}
  The next result provides an upper bound on the covering numbers of hypercubes.

  \begin{proposition}
  \label{prop:cover_hypercube}
  {\sc (Covering number of hypercubes)}
  Let $[0,R]^d$ be an hypercube of dimensionality $d\geq 1$
  whose side has length $R>0$. Then, 
  for all $\epsilon>0$, 
  we have that 
  \[
    \mathcal{N}_{\epsilon}([0,R]^d) \leq (\sqrt{d}R /2\epsilon )^d \vee 1.
  \]
  \end{proposition}

  \begin{proof}
  Observe first that since $[0,R]^d \subseteq B(c, \sqrt{d}R/2)$
  where $c$ denotes the center of the hypercube, then the result
  trivially holds for any $\epsilon\geq \sqrt{d}R/2$.
  Fix any $\epsilon < \sqrt{d}R/2 $,
  set $N_{\epsilon} = \lceil{\sqrt{d}R/2\epsilon }\rceil$ and
  define for all $I \in \{0, \dots,N_{\epsilon}-1  \}^d$ the 
  series 
  $H_{I} = I \times R / N_{\epsilon} + [0,R/N_{\epsilon}]^d$
  of $N_{\epsilon}^d$ hypercubes
  which cover $[0,R]^d :=\bigcup_{I \in \{0, \dots,N_{\epsilon}-1  \}^d} H_{I}$.
  Dentoting by $c_I$ the center of $H_I$ and
  since $\max_{x \in H_I} \norm{x-c_I}_2 \leq \epsilon$,
  it necessarily follows that $H_I \subseteq B(c_I, \epsilon)$
  which implies that
  $[0,R]^d \subseteq \bigcup_{I \in \{0, \dots,N_{\epsilon}-1  \}^d}
  B(c_I, \epsilon)$ and proves that $\mathcal{N}_{\epsilon}([0,R]^d)\leq 
  N_{\epsilon}^d \leq (\sqrt{d}R /2\epsilon )^d$.
  \end{proof}
  This result can be extended to any compact and convex set of $\R^d$
  as shown below.

  \begin{corollary}
  \label{coro:covering_number}
  {\sc (Covering number of a convex set)} 
  For any bounded compact and convex set $\X \subset \R^d$,
  we have that $\forall \epsilon>0$,
  \[
    \mathcal{N}_{\epsilon}(\X) \leq (\sqrt{d}\diam{\X}/\epsilon )^d \vee 1.
  \]
  \end{corollary}
\begin{proof}
  First, we show that 
  $\mathcal{N}_{\epsilon}(\X)
  \leq \mathcal{N}_{\epsilon}([0,2\diam{\X}]^d)$
  and then, we use the bound of Proposition \ref{prop:cover_hypercube} to conclude the proof.
  By definition of $\diam{\X}$, we know that there
  exists some $x \in \R^d$ such that $\X \subseteq x + [0, 2\diam{\X}]^d$.
  Hence, 
  we know from Proposition \ref{prop:cover_hypercube} that there exists a sequence
  $c_1, \ldots, c_{N_{\epsilon}}$ of 
  $N_{\epsilon}\leq \mathcal{N}_{\epsilon}([0,2\diam{\X}]^d)$
  points in ${[0,2\diam{\X}]}^d$ 
  forming an $\epsilon$-cover of $\X$:
  \begin{equation}
  \label{eq:cover_convex}
    \X \subseteq  [0,2\diam{\X}]^d
    \subseteq \bigcup_{i=1}^{N_{\epsilon}} B(c_i, \epsilon).
  \end{equation}
  However, we do not have the guarantee at this point that
  the centers $c_1, \dots c_{N_{\epsilon}}$ belong to $\X$.
  To build  an $\epsilon$-cover of $\X$, we project each of those centers on $\X$.
  More precisely, we  show that
  $
   \X \subseteq \bigcup_{i=1}^{N_{\epsilon}} B(\Pi_{\X}(c_i), \epsilon)
  $
  where 
  $\Pi_{\X}: x \in \R^d \mapsto \arg\min_{x' \in \X} \norm{x - x'}_2 \in \X$
  denotes the projection over the compact and convex set $\X$.
  Starting from (\ref{eq:cover_convex}), 
  it sufficient to show that $B(c_i,\epsilon) \cap\X \subseteq B(\Pi_{\X}(c_i), \epsilon)$,
   for all $i \in\{ 1,\dots, N_{\epsilon}\}$ to prove that
  \[
   \X \subseteq \bigcup_{i=1}^{N_{\epsilon}} B(c_i, \epsilon) \cap \X
   \subseteq \bigcup_{i=1}^{N_{\epsilon}} B(\Pi_{\X}(c_i), \epsilon).
  \]
  Pick any $c \in \{c_1,\dots, c_{N_{\epsilon}}\}$ and consider the following cases
  on the distance $\norm{c - \Pi_{\X}(c) }_2$ between the center and its projection:
  (i) if $\norm{c - \Pi_{\X}(c) }_2 =0$, then $c = \Pi_{\X}(c)$ and we have
  $B(c, \epsilon) \cap \X \subseteq B(\Pi_{\X}(c), \epsilon)$,
  (ii) If $\norm{c - \Pi_{\X}(c) }_2 > \epsilon$, then
   $\X \cap B(c, \epsilon) = \emptyset$, and we have
  $\X \cap B(c, \epsilon) \subseteq B(\Pi_{\X}(c), \epsilon)$. We now consider the
  non-trivial case where $\norm{c - \Pi_{\X}(c) }_2 \in (0,\epsilon)$.
  Pick any $x \in B(c, \epsilon) \cap \X$ and note that since
  $x \in B(c, \epsilon)$, then
  \begin{align*}
  \epsilon^2 &\geq \norm{x -c}_2^2 \\
  &= \norm{ x - \Pi_{\X}(c) + \Pi_{\X}(c)  - c}_2^2 \\
  & = \norm{x - \Pi_{\X}(c)}_2^2 + \norm{c - \Pi_{\X}(c)}_2^2
  + 2 \cdot \inner{x - \Pi_{\X}(c), \Pi_{\X}(c) - c  }
  \end{align*}
  which combined with the fact that
  $\norm{c - \Pi_{\X}(c)}_2^2\geq 0$ gives
    \begin{align}
    \label{eq:cover_conv}
   \norm{x - \Pi_{\X}(c)}_2^2 \leq   \epsilon^2 -
  2 \cdot \inner{x - \Pi_{\X}(c), \Pi_{\X}(c) - c  }.
  \end{align}
  We will simply show that the inner product $\inner{x - \Pi_{\X}(c) , \Pi_{\X}(c) - c  }$
  cannot be stricly negative 
  to prove that $\norm{x - \Pi_{\X}(c)}_2 \leq \epsilon$.
  Assume by contradiction that $\inner{x - \Pi_{\X}(c) , \Pi_{\X}(c) - c  }<0$.
  Since $\Pi_{\X}(c) \in \X$ and $x \in \X$, it follows
  the convexity of $\X$ implies that
  $\forall \lambda \in [0,1]$,
  $
  x_{\lambda} = \Pi_{\X}(c) + \lambda \cdot (x - \Pi_{\X}(c)   ) \in \X.
  $
  However, for all $\lambda \in (0,1)$ we have that
      \begin{align*}
  \norm{x_{\lambda} - c }_2^2 
  &= \norm{\Pi_{\X}(c) -  c +  \lambda \cdot (x - \Pi_{\X}(c)   ) }_2^2 \\
  & = \norm{\Pi_{\X}(c) -  c}_2^2 + \lambda^2 \norm{x - \Pi_{\X}(c)}_2^2
  + 2 \lambda \cdot \inner{ \Pi_{\X}(c) - c, x - \Pi_{\X}(c) } \\
  & = \norm{ \Pi_{\X}(c) -  c }_2^2 
  + \lambda \cdot (\lambda \norm{x - \Pi_{\X}(c)}_2^2  
  +2 \cdot \inner{ \Pi_{\X}(c) - c, x - \Pi_{\X}(c) }).
  \end{align*}
  Therefore, taking any
  $ 0<\lambda^{\star} <
  |\inner{\Pi_{\X}(c)-c, x-\Pi_{\X}(c) })|/ \norm{ \Pi_{\X}(c) -c }_2^2
  \wedge 1$ so that the second term of the right hand term of the previous equation is 
  strictly negative gives that 
  $
   \norm{x_{\lambda^{\star}} - c }_2^2  <  \norm{ \Pi_{\X}(c) -  c }_2^2
  $
  leads us to the following contradiction
    $
  \min_{x \in \X} \norm{x-c}_2 \leq \norm{x_{\lambda^{\star}} -c  }_2 
  < \norm{ \Pi_{\X}(c) -c  }_2 = \min_{x \in \X} \norm{x - c }_2.
   $
  Hence,
  $\inner{x - \Pi_{\X}(c) , \Pi_{\X}(c) - c  } \geq 0$
  and we deduce from (\ref{eq:cover_conv}) that
  $\X\cap B(c, \epsilon) \subseteq  B(\Pi_{\X}(c), \epsilon)$,
  which completes the proof.
  \end{proof}
  The next inequality will be useful to bound to bound volume of the intersection 
  of a ball and a convex set.

  \begin{lemma}
  \label{lem:ball_zab}
  {\it (From \cite{zabinsky1992pure}, see Appendix Section therein).}
   For any compact and convex set $\X \subset \R^d$ with
   non-empty interior,
   we have that for any $x^{\star} \in \X$ and $\epsilon \in (0,\diam{\X})$,
   \[
    \frac{\mu( B(x^{\star},\epsilon ) \cap \X) }{ \mu(\X)} \geq
    \left( \frac{\epsilon}{\diam{\X}} \right)^d.
   \]
  \end{lemma}

  \begin{proof}
  We point out that 
  a detailed proof of this result can be found in
  the Appendix Section of (\cite{zabinsky1992pure}).
  Nonetheless, we provide here a proof with less details for completeness.
  Introduce the similarity transformation  $S: \R^d \rightarrow \R^d$
  defined by
  \[
    S: x \mapsto x^{\star} + \frac{r}{\diam{\X}} (x-x^{\star})
  \]
  and let $S(\X):=\{S(x) : x \in \X \}$ be the image of $\X$ by $S$.
  Since $x^{\star} \in \X$ and
  $\max_{x \in \X}\norm{x-x^{\star}}_2 \leq \diam{\X}$ by definition,
  it follows from the convexity of $\X$ that $S(\X) \subseteq B(x^{\star},r)\cap \X$
  which implies that $\mu( B(x^{\star},r)\cap \X ) \geq \mu( S(\X))$.
  However, as $S$ is a similarity transformation conserves
  the ratios of the volumes before/after transformation, we thus deduce that
  \[
    \frac{\mu( B(x^{\star}, r) \cap \X )}{\mu( \X)}
    \geq \frac{\mu(S(\X))}{\mu(\X)}
    = \frac{\mu( S( B(x^{\star}, \diam{\X}) ) )}{ \mu( B(x^{\star}, \diam{\X}) )}
    = \frac{ \mu( B(x^{\star},r ) )   }{ \mu( B(x^{\star},\diam{\X}) ) }
  \]
   and the result follows using the fact that
  $\forall r \geq 0$,
  $\mu(B(x^{\star},r))=\pi^{d/2}r^{d}/\Gamma(d/2+1)$
  where $\Gamma(\cdot)$ stands
  for the standard gamma function.
  \end{proof}

  \begin{proposition}
  \label{prop:cvg_prs}
   {\sc (Pure Random Search)} Let $\X \subset \R^d$ be
   a compact and convex set with non-empty interior
   and let $f \in \Lip(k)$ be a $k$-Lipschitz functions defined on $\X$ for some $k\geq0$.
   Then, 
   for any $n \in \mathbb{N}^{\star}$ and
   $\delta \in(0,1)$, we have with probability at least $1-\delta$,
   \[
    \max_{x \in \X}f(x) -\max_{i=1 \dots n} f(X_i) 
    \leq k \cdot \diam{\X} \cdot \left(  \frac{\ln(1/\delta)}{n} \right)^{\frac{1}{d}}
   \]
   where $X_1, \dots, X_n$ denotes a sequence of $n$ independent copies of
   $X \sim \mathcal{U}(\X)$.
  \end{proposition}

  \begin{proof}
  Fix any $n\in\mathbb{N}^{\star}$ and
  $\delta \in (0,1)$, let $\epsilon = k \diam{\X} (\ln(1/\delta)/n )^{1/d}$ 
  be the value of the upper bound 
  and
  $\X_{\epsilon}=\{ x \in \X: f(x) \geq \max_{x \in \X}f(x) -\epsilon\}$
  the corresponding level set. 
  As the result trivially holds whenever $n \leq \ln(1/\delta)$, we consider
  that $n > \ln(1/\delta)$.
  Observe now that since $f \in \text{Lip}(k)$, then 
  for any $x^{\star} \in  \arg\max_{x \in \X}f(x)$, we have that
  $
    \X\cap B(x^{\star}, \epsilon/k) \subseteq \X_{\epsilon}
  $
  since
  $
    ~|f(x) - f(x^{\star})| \leq  k \cdot \norm{x-x^{\star}}_2
    = \epsilon
  $
  for all $x \in B(x^{\star}, \epsilon/k) \cap \X$.
  Therefore, by picking any $x^{\star} \in  \arg\max_{x \in \X}f(x)$, one gets
  \begin{align*}
  \P\left( \max_{i=1 \dots n}f(X_i) \geq \max_{x \in \X}f(x) -\epsilon \right)  &
  = \P\left( \bigcup_{i=1}^n \{X_i \in \X_{\epsilon} \} \right) 
  & \text{(def.~of } \X_{\epsilon}) \\
  & = 1 - \P\left( X_1 \notin \X_{\epsilon} \right)^n 
  & \text{~(i.i.d.~r.v.)}\\
  & \geq 1 - \P(X_1 \notin \X\cap B(x^{\star}, \epsilon/k))^n
  & \!\!\!\!\!\!\!\!\!\!\!(\X\cap B(x^{\star}, \epsilon/k)\subseteq \X_{\epsilon})\\
   & = 
   1- \left(1 - \left(\frac{\mu(\X\cap B(x^{\star}, \epsilon/k))}{\mu(\X)}\right)^d \right)^n
   & (X_1\sim\mathcal{U}(\X))\\
    & \geq 1 - \left(1 - \left( \frac{\epsilon}{k\diam{\X}}\right)^d \right)^n 
    & \text{(Lemma \ref{lem:ball_zab})}\\
    & = 1- \left( 1 - \frac{\ln(1/\delta)}{n} \right)^n 
    & \text{(def.~of~}\epsilon) \\
    & \geq 1 -\delta. & (1+x\leq e^x)
  \end{align*}
  \end{proof}

\section{Proofs of  Section \ref{sec:lipopt}}

In this section, we provide the proofs of Propositions \ref{prop:consistency_equivalence},
\ref{prop:unbounded_error}, \ref{prop:minimax} and
Example \ref{ex:PRS}.\\

\sloppy    
\noindent \textbf{Proof of proposition \ref{prop:consistency_equivalence}.}
  ($\Leftarrow$) Let $A$ be any global optimization 
  algorithm such that
  $\forall f \in \bigcup_{k\geq0}\Lip(k)$,
    $
    \sup_{x \in \X} \min_{i=1\dots n} \norm{X_i - x}_2 \xrightarrow{p}0.
  $
  Pick any $\epsilon>0$, any $f \in \bigcup_{k\geq0}\Lip(k)$
  and let  $\X_{\epsilon} = \{x \in \X: f(x) \geq \max_{x \in \X}f(x) -\epsilon  \}$ 
  be the corresponding level set.
  As $\X_{\epsilon}$ is non-empty, 
  there necessarily exists some
  $ x_{\epsilon} \in \X$ and $r_{\epsilon}>0$ such that
  $B(x_{\epsilon}, r_{\epsilon})\cap\X \subseteq \X_{\epsilon}$.
  Thus, if 
  $X_1,\dots, X_n$ denotes a sequence a sequence of $n$ evaluation points generated 
  by $A$ over $f$,
  we directly obtain 
  from the convergence in probability of the mesh grid that
  \begin{align*}
  \P\left(  \max_{x \in \X}f(x) -\max_{i=1\dots n}f(X_i) >\epsilon \right)  
   &= \P\left(\bigcap_{i=1}^n \left\{ X_i \notin \X_{\epsilon}  \right\} \right) \\
   & \leq \P\left(\bigcap_{i=1}^n\left\{ X_i \notin B(x_{\epsilon}, r_{\epsilon}) \right\}\right)\\
   &= \P\left( \min_{i=1\dots n} \norm{X_i - x_{\epsilon}}_2 > r_{\epsilon} \right)\\
   & \leq \P\left(\sup_{x \in \X}\min_{i=1\dots n} \norm{X_i - x}_2 > r_{\epsilon}
    \right)
    \xrightarrow[n \to \infty]{} 0.
  \end{align*} 
  
  \noindent ($\Rightarrow$) Let $A$ be any global optimization algorithm 
  consistent over the set of Lipschitz functions and assume by contradiction that
  there exists some $f^{\star}\in\bigcup_{k\geq0}\Lip(k)$ such that
  $
   \sup_{x \in _X} \min_{i=1 \dots n} \norm{x - X_i}_2  \overset{p}{\nrightarrow} 0.
   $
   The implication is proved in two steps: 
  first, we show that there exists a ball $B(c^{\star},\epsilon)$
  for some $c^{\star}\in\X$ which is 
  almost never hit by the algorithm and second, we build a Lipschitz function which admits
  its maximum over this ball.\\
  
  \noindent {\it First step.}
  Let $\{X_i\}_{i \in \mathbb{N}^{\star}}$ be a sequence of evaluation points
  generated by $A$ over $f^{\star}$.
  Observe first that since for all $\epsilon >0$, the series $n \in \mathbb{N}^{\star} \mapsto 
  \P( \sup_{x \in _X} \min_{i=1 \dots n} \norm{x - X_i}_2 > \epsilon)$
  is non-increasing, then the contradiction assumption necessarily implies that
  \begin{equation}
  \label{prop:cover_eq2}
  \exists \epsilon_1, \epsilon_2>0 \text{~such that~} \forall n \in \mathbb{N}^{\star},
  ~\P\left(\sup_{x \in \X} \min_{i=1\dots n} \norm{x - X_i}_2
  > \epsilon_1 \right) > \epsilon_2.
  \end{equation}
  Consider now any sequence $c_1, \dots, c_{N_{1}} $ 
  of ${N_{1}} =\mathcal{N}_{\epsilon_1}(\X)$  points in $\X$ defining 
  an $\epsilon_1$-cover of $\X$
  and suppose by contradiction that 
  \begin{equation*}
   \forall c \in \{c_1, \dots, c_{N_1} \},~
   \exists n_{c} \in \mathbb{N}^{\star}\text{~such that~}
   \P\left( \bigcap_{i=1}^{n_c}  \{X_i \notin B(c,\epsilon_1)\cap \X \} \right)
   \leq \frac{\epsilon_2}{2  N_1}
  \end{equation*}
  which gives by setting $N_2 =\max_{c \in \{c_1, \cdots, c_{N_{1}} \} } n_c $ that
  \[
   \forall c \in \{c_1, \dots, c_{N_{1}} \},~
   \P\left( \bigcap_{i=1}^{N_2}  \{X_i \notin B(c,\epsilon)\cap \X \} \right)
   \leq \frac{\epsilon_2}{2 N_1}.
  \]
  However, as $c_1,\dots, c_{N_1}$ form an $\epsilon_1$-cover  of $\X$, it follows that
  \begin{align*}
   \P\left( \sup_{x \in \X} \min_{i=1 \dots N_2} \norm{x-X_i}_2 \leq \epsilon_1
   \right) & \geq 
   \P\left(  \bigcap_{j=1}^{N_{1}} \bigcup_{i=1}^{N_2} 
   \{ X_i \in B(c_j,\epsilon_1) \cap \X \} \right)  \\
   & = 1 - \P\left( \bigcup_{j=1}^{N_1} \bigcap_{i=1}^{N_2} \{X_i \notin B(c_j,\epsilon_1) \cap \X \}   \right) \\
   & \geq 1 - \sum_{j=1}^{N_1} \P\left(  \bigcap_{i=1}^{N_2} 
   \{ X_i \notin B(c_j,\epsilon) \cap \X \}  \right) \\
   & \geq 1 - N_1
   \times \frac{\epsilon_2}{2 N_1} \\
   & = 1 - \frac{\epsilon_2}{2}
  \end{align*}
  which contradicts (\ref{prop:cover_eq2}). Hence, we deduce that  
  \[
   \exists c^{\star} \in \{c_1, \dots, c_{N_{\epsilon}} \}
   \text{~such that~}
   \forall n \in \mathbb{N}^{\star},~ 
   \P\left( \bigcap_{i=1}^{n}  \{X_i \notin B(c^{\star},\epsilon_1)\cap \X \}\right) \geq 
   \frac{\epsilon_2}{2N_1}.
  \]
  {\it Second Step.} 
  Based on this center $c^{\star}\in\X$, one can  introduce the function 
  $\tilde{f}: \X \mapsto \R$ defined for all $x \in \X$ by
  \[
   \tilde{f}(x) =
   \begin{cases}
   f^{\star}(x) + 
    3\left(1 - \frac{\norm{c^{\star} - x}_2}{\epsilon_1} \right)
    \times (\max_{x \in \X}f^{\star}(x)  - \min_{x \in \X}f^{\star}(x))
    & \text{~if~}x \in B(c^{\star},\epsilon_1)\\
    f^{\star}(x) & \text{~otherwise}
   \end{cases}
  \]
  which is maximized over $B(c^{\star}, \epsilon_1)$ and 
  Lipschitz continuous as both $f^{\star}$ 
  and $x \mapsto \norm{c^{\star} - x}_2$ are Lipschitz.
  However, since $\tilde{f}$ and $f^{\star}$ can not be distinguished over
  $\X/B(c,\epsilon_1)$,
  we have that $\forall n \in \mathbb{N}^{\star}$,
  \begin{align*}
   \P\left(\max_{x \in \X}\tilde{f}(x) - \max_{i=1\dots n}\tilde{f}(X'_i) > \max_{x \in \X}f(x)  \right)
    & \geq \P\left( \bigcap_{i=1}^n \{X'_i \notin B(c,\epsilon_2) \cap \X \} \right)\\
    &= \P\left( \bigcap_{i=1}^n \{X_i \notin B(c,\epsilon_2)\cap\X \} \right) \\
    &\geq \epsilon_2/(2N_1)\\
    & >0
  \end{align*}
  where $X'_1, \dots,  X_n'$ denotes a sequence of evaluation points
  generated by $A$ over $\tilde{f}$, and we deduce
  that there exists $\tilde{f} \in \bigcup_{k\geq0}\Lip(k)$
  such that 
  $
  \max_{i=1\dots n}\tilde{f}(X'_i)  \overset{p}{\nrightarrow} \max_{x \in \X}\tilde{f}(x).
  $
  Hence, it contradicts the fact that $A$ is consistent over $\bigcup_{k\geq 0}\Lip(k)$
  and we deduce
  that, necessarily, $\sup_{x \in \X}\min_{i=1\dots n} \norm{X_i-x} \overset{p}{\rightarrow} 0$
  for all $f \in \bigcup_{k\geq0}\Lip(k)$.
  \hfill\(\Box\)

  ~\\
   {\bf Proof of Example \ref{ex:PRS}.} 
   Fix any $n \in \mathbb{N}^{\star}$, set $\delta \in (0,1)$,  define 
   $\epsilon = 
   \diam{\X}\cdot \left(  (\ln(n/\delta ) +d \ln(d))/n \right)^{1/d}$
   and let $X_1, \dots, X_n$ be a sequence of $n$ independent copies of $X \sim \mathcal{U}(\X)$.
   Since the result trivially holds whenever $(\ln(n/\delta ) +d \ln(d))/n\geq 1$, 
   we consider the case where $(\ln(n/\delta ) +d \ln(d))/n<1$.
   From Proposition \ref{coro:covering_number}, we know
   that there exists a sequence $x_1, \dots, x_{N_{\epsilon}} $ of 
   $N_{\epsilon} = \mathcal{N}_{\epsilon}(\X)$ points in $\X$ 
   such that $\X \subseteq \bigcup_{j=1}^{N_{\epsilon}} B(x_j, \epsilon)$.
   Therefore, using the bound on the covering number 
   $\mathcal{N}_{\epsilon}(\X)$ of Corollary \ref{coro:covering_number}, 
   we obtain that
   \begin{align*}
    \P\left( \sup_{x \in \X} \min_{i=1 \dots n} \norm{x-X_i}_2 \leq \epsilon
     \right) 
     & \geq \P\left(  \bigcap_{j=1}^{N_{\epsilon}}
     \bigcup_{i=1}^n \{X_i \in B(x_j, \epsilon)\cap \X \} \right) \\
     & = 1 - \P\left(
     \bigcup_{j=1}^{N_{\epsilon}}
     \bigcap_{i=1}^n \{X_i \notin B(x_j, \epsilon)\cap \X \} \right) \\
     & \geq 1 - \sum_{j=1}^{N_{\epsilon}}
     \P\left( \bigcap_{i=1}^n \{X_i \notin B(x_j, \epsilon)\cap \X \}  \right) \\
     & \geq 1 - 
     N_{\epsilon} \times
     \max_{j=1\dots N_{\epsilon}} \P(X_1 \notin B(x_j, \epsilon)\cap \X)^n \\
     & = 1 - N_{\epsilon}
     \times \max_{j=1\dots N_{\epsilon}} 
     \left( 1-  \frac{ \mu(\X \cap B(x_j, \epsilon)) }{\mu(\X)} \right)^n \\
     & \geq 1 -N_{\epsilon} \times
     \left(1 - \left( \frac{\epsilon}{\diam{\X}} \right)^d \right)^n \\
     & \geq  1 - \left( \frac{\sqrt{d}\diam{\X}}{\epsilon} \right)^d
     \times
     \left(1 - \left( \frac{\epsilon}{\diam{\X}} \right)^d \right)^n \\
     & \geq 1 -\delta
   \end{align*}
   and the proof is complete.\hfill\(\Box\) \\

  \noindent {\bf Proof of Proposition \ref{prop:unbounded_error}.}
  The proof heavily builds upon the arguments used in the proof of the Theorem 1
  in (\cite{bull2011convergence}). 
  Pick any algorithm $A\in\A$ and any constant $C>0$.
  Fix any $n\in\mathbb{N}^{\star}$ and $\delta \in (0,1)$
  and set $N_{\delta} = \lceil (n/\delta)^{1/d}\rceil$. 
  By definition of $\rad{\X}$, we know
  there exists some $x \in \X$ such that $x +[0,2\rad{\X}/\sqrt{d}]^d \subseteq \X$.
  One can then define
  for all $I \in \{1, \dots, N_{\delta}\}^d$, the centers $c_I$ of the hypercubes $H_I$
  whose side are equal to $D = 2\rad{\X}/ (\sqrt{d}N_{\delta})$
  and cover $\X$, {\it i.e.,} $\bigcup_I H_I = x +[0,2\rad{\X}/\sqrt{d}]^d \subseteq \X$.
  Now, let $X_1,\dots,X_n$ be a sequence of $n$ evaluation points generated
  by the algorithm $A$ over the constant function $f_0: x \in \X \mapsto 0$
  and define for all $I \in \{1, \dots, N_{\delta}\}^d $ the event
  \[
   E_I = \bigcap_{i=1}^n \left\{X_i \notin \text{Int}(H_I)\right\}.
  \]
  As the interiors of the $N_{\delta}^d$ hypercubes are disjoint
  and we have $n$ points, 
  it necessarily follows that
  \begin{align*}
   N_{\delta}^d  \times \max_{I} ~\!\P(E_I) 
   \geq \sum_{I} \P(E_I)
    = \esp{\sum_I \indic{E_I} }
    \geq N_{\delta}^d - n.
  \end{align*}
  Hence, there exsits some fixed $I^{\star}$ only depending on $\A$
  which maximizes the above probability and thus satisfies
  \[
   \P(E_{I^{\star}}) \geq  \frac{N_{\delta}^d - n}{N_{\delta}^d}
   =  1 - \frac{n}{\lceil(n/\delta)^{1/d}\rceil^d}
   \geq 1 -\delta.
  \]
  Now, using the center $c_{I^{\star}}$ of the hypercube $H_{I^{\star}}$,
  one can then introduce the function $\tilde{f} \in \bigcup_{k\geq0}\Lip(k)$ 
  defined for all $x \in \X$ by
  \[
   \tilde{f}(x) =
   \begin{cases}
    C \times ( 1- 2 \norm{c_{I^{\star}} -x }_2/D) & \text{if~}
   \norm{c_{I^{\star}} -x }_2  \leq D/2\\
    0 & \text{otherwise.}
   \end{cases}
  \]
  However, since the functions $\tilde{f}$ and $f_0$ can not be distinguished over
  $\X/ H_{I^{\star}}$,
  we have that
  \begin{align*}
   \P\left(\max_{x \in X}\tilde{f}(x) - \max_{i=1\dots n}\tilde{f}(X'_i) \geq C \right)
    \geq \P\left( \bigcap_{i=1}^n \{X'_i \notin \text{Int}(H_{I^{\star}}) \} \right) 
    = \P(E_{I^{\star}})
    \geq 1- \delta
  \end{align*}
  where $X'_1,\dots,X_n'$ denotes a sequence 
  of evaluation points generated by $\A$ over $\tilde{f}$,
  which proves the result.\hfill\(\Box\)

  ~\\
  {\bf Proof of Proposition \ref{prop:minimax}.}
  {\bf (Lower bound).}
  Pick any $n \in \mathbb{N}^{\star}$ and set 
  $D= 2\rad{\X}/(\sqrt{d}\lceil(2n)^{1/d} \rceil )$.
  It can easily be shown by reproducing the same steps as in the
  proof of Proposition \ref{prop:unbounded_error} with $\delta$ set to $1/2$,
  that for any global optimization algorithm $A$,
  there exists a function $\tilde{f}_A\in \text{Lip}(k)$
  defined by
    \[
   \tilde{f}_A(x) =
   \begin{cases}
    kD/2 -k\cdot  \norm{c_{A} -x }_2 & \text{if~}
   \norm{c_{A} -x }_2  \leq D/2\\
    0 & \text{otherwise,}
   \end{cases}
  \]
  for some center $c_A \in \X$ only depending on $A$, for which we have
  $
   \P(\max_{x \in X}\tilde{f}_A(x) 
   - \max_{i=1\dots n}\tilde{f}_A(X_i) \geq k\cdot D/2)
   \geq 1/2
  $
  where 
  $X_1,\dots,X_n$ is a sequence of $n$ evaluation points generated
  by $A$ over $\tilde{f}_A$.
  Therefore, using the definition of the supremum and Markov's inequality
  gives that $\forall A \in \A$:
  \begin{align*}
  \sup_{f \in \Lip(k)}\esp{\max_{x \in \X}f(x) - \max_{i=1\dots n}f(X_i)   }  
  & \geq \esp{\max_{x \in \X}\tilde{f}_A(x) - \max_{i=1\dots n}\tilde{f}_A(X_i)   } \\
  & \geq  \frac{kD}{2} \times
   \P\left(\max_{x \in X}\tilde{f}_A(x) - \max_{i=1\dots n}\tilde{f}_A(X_i) 
   \geq k\cdot \frac{D}{2} \right) \\
  & \geq k \cdot \frac{\rad{\X}}{8\sqrt{d}} \cdot n^{-\frac{1}{d}}.
  \end{align*}
  As the previous inequaliy holds true for any algorithm $A$,
  the proof is complete.\\
  
  \noindent {\bf (Upper bound).} 
  Sequentially  using the fact that (i)
  the infinimum minimax loss taken over all algorithms
  is necessarily upper bounded by the loss suffered by a Pure Random Search,
  (ii) for any positive random variable,
  $\esp{X} = \int_{t=0 }^{\infty} \P( X \geq t)  \text{d}t$,
  (iii) Proposition \ref{prop:cvg_prs}
  and (iv) the change of variable $u=n(t/\diam{\X})^{1/d}$,
  we obtain that
  \begin{align*}
      \inf_{A \in \A} \sup_{f \in \text{Lip}(k)}
   \esp{ \max_{x \in \X}f(x) -\max_{i=1 \dots n}f(X_i) } 
    & \leq \sup_{f \in \text{Lip}(k)}
   \esp{ \max_{x \in \X}f(x) -\max_{i=1 \dots n}f(X'_i) }  \\
    &\leq \int_{0}^{\infty} \exp\left\{ -n ( t/k\cdot \diam{\X})^{1/d} \right\}
    \text{d}t \\
     & =  k\cdot \diam{\X} \cdot n^{-d} \cdot d \cdot 
     \int_{0}^{\infty} u^{d-1} e^{-u}  \text{d}u\\
     & = k\cdot \diam{\X} \cdot n^{-d} \cdot d \cdot \Gamma(d)
  \end{align*}
    where $X'_1, \dots, X'_n$ denotes a sequence of $n$ independent copies of $X' \sim 
  \mathcal{U}(\X)$ and
  $\Gamma(\cdot)$ the Euler's Gamma function. 
  Recalling that $\Gamma(d)= (d-1)!$ 
  for all $d \in \mathbb{N}^{\star}$ completes the proof.\hfill\(\Box\)

  \section{Proofs of Section \ref{sec:lipopt}}
  
  In this section, 
  we provide the proofs for Lemma \ref{lem:potential}, Proposition 
  \ref{prop:consistency_lipopt}, Proposition \ref{prop:fasterprs},
  Corollary \ref{prop:upperlipopt}, Proposition \ref{prop:limit_lipopt},
  Theorem \ref{th:fast_rates} and Theorem \ref{th:lower_bound}.\\

  \noindent {\bf Proof of Lemma \ref{lem:potential}.} 
  The first implication ($\Rightarrow $) is a direct 
  consequence of the definition of $\X_{k,t}$.
  Noticing that the function
  $\hat{f}:x \mapsto \min( \max_{i=1\dots t}f(X_i) ,\min_{i=1\dots t}f(X_i) +k \norm{x-X_i}_2)$
  belongs to $\F_{k,t}$ and that
  $\arg\max_{x \in \X}\hat{f}(x)= \{x \in \X: 
  \min_{i=1\dots t } f(X_i) + k\norm{x-X_i}_2\geq \max_{i=1\dots t}f(X_i) \}$
  proves the second implication.\hfill\(\Box\)
  
  ~\\
  \noindent {\bf Proof of Proposition \ref{prop:consistency_lipopt}.} 
  Fix any $f \in \Lip(k)$, pick any $n\in \mathbb{N}^{\star}$, set
  $\epsilon>0$ and let
  $\X_{\epsilon} = \{ x\in \X: f(x) \geq \max_{x \in \X} f(x) -\epsilon\}$ be
  the corresponding level set.
  Denoting by $X_1', \dots, X_n'$ a sequence of $n$ random 
  variable uniformly distributed over $\X$ and 
  observing that $\mu(\X_{\epsilon})>0$,
  we directly obtain from Proposition \ref{prop:fasterprs} that
  \begin{align*}
   \P\left(\max_{x \in \X}f(x) - \max_{i=1\dots n}f(X_i) > \epsilon \right)
   &\leq 
   \P\left(\max_{x \in \X}f(x) - \max_{i=1\dots n}f(X'_i) > \epsilon \right)\\
   &=     \P\left(   
    \bigcap_{i=1}^n \{ X'_i \notin \X_{\epsilon}\}
    \right) \\
   &\leq 
    \left( 1 - \frac{\mu(\X_{\epsilon})}{\mu(\X)}   \right)^n
    \xrightarrow[n \to \infty]{} 0.
  \end{align*}
  \hfill\(\Box\)
  
  ~\\
  {\bf Proof of Proposition \ref{prop:fasterprs}.} 
  The proof is similar to the one of Proposition 12
  in (\cite{malherbe2016ranking}).\hfill\(\Box\)
  
  ~\\
  {\bf Proof of Corollary \ref{prop:upperlipopt}.} 
  Combining Proposition \ref{prop:fasterprs} and
  Proposition \ref{prop:cvg_prs} stated at the begining of the Appendix Section
  gives the result.\hfill\(\Box\)
  
  ~\\
  {\bf Proof of Proposition \ref{prop:limit_lipopt}.}
  Fix any $\delta \in (0,1)$, set $n \in \mathbb{N}^{\star}$ 
  and let $r_{\delta,n } = \rad{\X} ( \delta/n)^{\frac{1}{d}}$ be the value
  of the lower bound divided by $k$.
  As $\rad{\X}>0$, there necessarily
  exists some point $x^{\star}\in \X$ such that $B(x^{\star}, \rad{\X}) \subseteq \X$.
  Based on this point,
  one can then introduce the function $\tilde{f}\in\text{Lip}(k)$ defined 
  for all $x \in \X$ by
    \begin{equation*}
      \tilde{f}(x) =
      \begin{cases}
	\   k\cdot r_{\delta, n}
	- k\cdot \norm{x - x^{\star}}_2 &
	\text{~if~} x \in B(x^{\star}, r_{\delta, n}) \ \   \\
	\  0 & \text{~otherwise.} \
      \end{cases}  
    \end{equation*} 
   Denoting now by $X_1, \dots, X_n$ a sequence of $n$ evaluation points
   generated by LIPO tuned with a parameter $k$ over $\tilde{f}$
   and observing that (i) $X_1$ is uniformly distributed over $\X$
   and (ii) $X_{i+1}$ is also uniformly distributed over $\X$ 
   for $i\geq 1$ as soon as only constant evaluations have been recorded
   ({\it i.e.}~$\X_{k,i+1} = \X$ on the event 
   $\bigcap_{t\leq i}\{X_t \notin  B(x^{\star},r_{\delta, n} ) \}$),
   we have that
     \begin{align*}
   \P\left( \max_{x \in \X}\tilde{f}(x) 
   - \max_{i=1\dots n}\tilde{f}(X_i) \geq k\cdot r_{\delta,n} \right)
    & \geq  \P\left(\bigcap_{i=1}^n \{ X_i \notin B(x^{\star},r_{\delta, n} )\} \right)\\
    & = \Bigg[ \P(X_1 \notin B(x^{\star},r_{\delta, n} ) )\times \\
    & ~~~\prod_{i=1}^{n-1} 
    \P\!\left(X_{i+1} \notin B(x^{\star},r_{\delta, n} ) \mid
     \bigcap_{t=1}^i\{X_t \notin B(x^{\star},r_{\delta, n} ) \} \right)\!
     \Bigg] \\ 
     &= \left( 1- \frac{\mu(B(x^{\star}, r_{\delta, n}) \cap \X) }{\mu(\X)}  \right)^n  \\
     & \geq \left(1 - \left(\frac{r_{\delta, n}}{\rad{\X}}\right)^d \right)^n \\
    & =\left( 1 - \frac{\delta}{n}\right)^n \\
    & \geq 1-\delta.
  \end{align*}
   \hfill\(\Box\)
   
%

  ~\\
  \noindent {\bf Proof of Theorem \ref{th:fast_rates}.}
  Pick any $n \in \mathbb{N}^{\star}$, fix any $\delta \in (0,1)$ and 
  let $X_1, \dots, X_n$ be a sequence of $n$ evaluation points
  generated by the LIPO algorithm over $f$ after $n$ iterations.
  To clarify the proof, we set some specific notations:
  let $D=\max_{x \in \X}\norm{x-x^{\star}}_2$, set
  \begin{equation*}
   M = \begin{cases}
     \left\lfloor 
     \left( \frac{c_{\kappa}}{8k} \right)^d \cdot 
     \frac{n}{ \ln(n/\delta) + 2(2\sqrt{d})^d  }
     \right\rfloor
     &  \text{~~if~~} \kappa=1 \\
     ~&\\
    \left\lfloor \frac{1}{\ln(2)d(\kappa-1)}
    \ln\left( 1+
    \left( \frac{c_{\kappa}D^{\kappa}}{8kD} \right)^d
    \frac{n(2^{d(\kappa-1)}-1)}{\ln(n/\delta) + 2(2\sqrt{d})^d}   \right) \right\rfloor
    &  \text{~~otherwise,~~}
   \end{cases}
  \end{equation*}
  define for all  $m \in \{1\dots M \}$  the series of integers:
  \[
      N_m := \left\lceil \sqrt{d} \cdot \left( 
    \frac{8 k D }{c_{\kappa} D^{\kappa}}  \right)
    \cdot 2^{m(\kappa-1)} \right\rceil^d
  \text{~~~and~~~}
     N'_m := \left\lceil \ln(M/\delta) 
    \cdot \left( \frac{  8 k D}{c_{\kappa} D^{\kappa}}\right)^d 
    \cdot 2^{md(\kappa-1)}
     \right\rceil
  \]
  and let $ \tau_0, \dots, \tau_M$ be the series of stopping times
  initialized by  $\tau_0= 0$ and defined for all $m\geq 1$ by
  \begin{equation*}
    \tau_{m}  := \inf\left\{t \geq \tau_{m-1} \mid 
    \sum_{i=\tau_{m-1}+1}^t \indic{X_i 
    \in B(x^{\star}, 2\cdot D\cdot2^{-m})} = N'_{m} \right\}.
  \end{equation*}
  The stopping time $\tau_m$ correspond to the time after $\tau_{m-1}$
  where we have recorded at least $N'_m$ random evaluation points
  inside the ball $B(x^{\star},2 \cdot D \cdot 2^{-m})$.
  To prove the result, we  show that each of the following events:
    \begin{align*}
      E_m := 
      \left\{  \max_{i=1 \dots \tau_m} f(X_i)
   \geq \max_{x \in \X}f(x)  - \frac{c_{\kappa}}{2} 
   \cdot \left( \frac{D}{ 2^{m} } \right)^{\kappa} \right\}
      \cap \left\{ \tau_m \leq  N'_1  
      + \sum_{l=1}^{m-1} \left( N'_{l+1}  + N_l \right) \right\}.
    \end{align*}
  holds true with probability at least $1-\delta/M$
  on the event $\bigcap_{l=1}^{m-1}E_l$ for all $m\in \{2, \dots, M \}$
  so that:
  \begin{equation}
  \label{eq:final}
   \P\left( E_M \right) \geq 
   \P(E_1) \times \prod_{m=1}^{M-1} \P\left( E_{m+1} | \bigcap_{l=1}^mE_l\right)
   \geq \left(1 - \frac{\delta}{M} \right)^M \geq 1- \delta
  \end{equation}
  that will leads us to the result by analyzing $E_M$.\\

 
  \noindent {\bf Analysis of $\P(E_1)$.} 
  Observe first that since $\X \subseteq B(x^{\star},D)$, then $\tau_1=N_1'$.
  Using now the fact 
  that (i) the algorithm is faster than a Pure Random Search 
  (Proposition \ref{prop:fasterprs})
  and (ii) the bound of Proposition \ref{prop:cvg_prs},
  we directly get that with probability at least $1-\delta/ M$,
  \begin{align*}
   \max_{x \in \X}f(x) - \max_{i=1\dots \tau_1 } f(X_i)
   & \leq k
  \cdot 2D \cdot \left( \frac{\ln(M/\delta)}{N'_1} \right)^{\frac{1}{d}} \\
   &\leq k\cdot 2D \cdot 
  \left( 
  \frac{\ln(M/\delta)}{\ln(M/\delta) 2^{d(\kappa-1)}}  
  \left( \frac{c_{\kappa}D^{\kappa} }{8kD} \right)^d
   \right)^{\frac{1}{d}}  \\
   &= \frac{c_{\kappa}}{2} \cdot \left(\frac{D}{2}\right)^{\kappa}
  \end{align*}
  which proves that $\P(E_1) \geq 1-\delta/M$.\\

  \noindent {\bf Analysis of $\P(E_{m+1}|\cap_{l=1}^{m} E_l )$.}
  To bound this term, we use (i) a deterministic covering argument
  to control the stopping time $\tau_{m+1}$ 
  (Lemma \ref{lem:covering} and Corollary \ref{coro:stopping})
  and 
  (ii) a stochastic argument to bound the maximum $\max_{i=1\dots \tau_{m+1}}f(X_i)$
  (Lemma \ref{lem:bonus2} and Corollary \ref{eq:coro}).
  The following lemma states that after $\tau_m$ and 
  on the event $E_m$  there will be at most 
  $N_m$ evaluation points that will fall inside the area 
  $B(x^{\star}, 2D\cdot 2^{-m}) /B(x^{\star}, D\cdot2^{-m})$.\\
  
  \begin{lemma}
  \label{lem:covering}
    For all $ m \in \{ 1,\dots, M-1\}$, 
    we have on the event $E_m$,
  \[
    \sum_{t=\tau_m+1}^{n} \indic{ X_t 
    \in B(x^{\star}, 2D\cdot 2^{-m}) /B(x^{\star}, D\cdot2^{-m})  }
    \leq N_m.
  \]
  \end{lemma}

  \begin{proof}
  Fix $m \in \{1, \dots, M -1\}$ and assume that 
  $E_m=\{ \max_{i=1\dots \tau_m}f(X_i) \geq 
  \max_{x \in \X}f(x) - c_k/2 \cdot (D/ 2^{m})^{\kappa}\}$ holds true.
  Setting $N = \lceil \sqrt{d} 8k D  2^{m(\kappa-1)} /(c_kD^{\kappa}) \rceil $
  and observing that $B(x^{\star}, 2D\cdot 2^{-m})
  \subseteq x^{\star} +2D\cdot 2^{-m} \cdot [-1,+1]^d$,
  one can then introduce the sequence $H_I$, with $I \in \{1, \dots, N \}^d$,
  of the $N^d = N_m$ hypercubes  
  whose side have length $4D\cdot 2^{-m}/N$ and
  cover $x^{\star} +2D\cdot2^{-m} \times [-1,+1]^d$, so that
  \[
  B(x^{\star}, 2D\cdot 2^{-m}) /B(x^{\star}, D\cdot 2^{-m})
  \subseteq  x^{\star} +2D\cdot2^{-m} \cdot [-1,+1]^d = \bigcup_I H_I.
  \]
  Based on these hypercubes, one can define the set 
    \[
      I_{t} = \{ I \in \{1, \dots, N \}^d: 
      H_I \cap 
      B(x^{\star}, 2D\cdot 2^{-m})/B(x^{\star}, 2D\cdot 2^{-m})
      \cap \X_{k,t}
      \neq \emptyset  \}
  \]
  which contains the indexes of the hypercubes that still intersect
  the set of potential maximizers $\X_{k,t}$ at time $t$ and the target area
  $B(x^{\star}, 2D\cdot 2^{-m})/B(x^{\star}, 2D\cdot 2^{-m})$.
  We show by contradiction that there cannot be more than $N^{d}=N_m$
  evaluation points falling inside this area, otherwise it would be empty.
  Suppose that, after $\tau_m$, there exists a sequence
  \[
    \tau_m < t_1 < t_2 < \dots < t_{N^{d}+1} \leq n
  \]
  of $N^{d}+1$ strictly increasing indexes for which the evaluation points 
  $X_{t_j}$, $j\geq 1$,
  belong to the target  area, {\it i.e,}
  $$
  \forall j \in \{ 1, \dots, N^d+1\},~
  X_{t_j} \in  B(x^{\star}, 2D\cdot 2^{-m}) /B(x^{\star}, D\cdot 2^{-m}).
  $$
  Fix any $j\geq 1$ and observe that since $X_{t_j}\notin B(x^{\star},D\cdot 2^{-m} )$,
  then we have from Condition \ref{cond} that
  $
  \text{(i)} ~~f(X_{t_j}) < \max_{x\in\X}f(x) - c_k \cdot (D\cdot 2^{-m})^{\kappa}.
  $
  Moreover, as
  $X_{t_j} \in \X_{k,t_j-1} \cap B(x^{\star}, 2D\cdot 2^{-m}) /B(x^{\star}, D\cdot 2^{-m})$,
  it necessarily follows from the definition of the algorithm that
  $
  \text{(ii) there  exists an index~} I^{\star} \in I_{t_j-1}
  \text{~such that~} X_{t_j} \in H_{I^{\star}}.
  $
  Therefore, combining (i) and (ii) with $E_m$, gives that
  $\forall x \in H_{I^{\star}}$:
  \begin{align*}
    f(x) &\leq f(X_{t_j}) + k \cdot \norm{X_{t_j} -x }_2 & (f \in \text{Lip}(k))\\
            & \leq f(X_{t_j})
    + k \cdot \max_{(x,x') \in H_I^2} \norm{x-x'}_2 &  ((X_{t_j},x) \in H_I^2) \\
        & = f(X_{t_j})
    + k \cdot \sqrt{d}\cdot 4D \cdot 2^{-m}/N &  (\text{def.~of~}H_I)\\
        & \leq f(X_{t_j})
    + \frac{c_k}{2} \cdot (D\cdot2^{-m})^{\kappa} &  (\text{def.~of~}N)\\
    & < \max_{x\in\X}f(x) - c_{\kappa} \cdot (D\cdot 2^{-m})^{\kappa} 
    + \frac{c_k}{2} \cdot (D\cdot2^{-m})^{\kappa}
    & \text{(i)}\\
    & \leq \max_{i=1\dots \tau_m}f(X_i) & (E_1) \\
    & \leq \max_{i=1\dots t_j}f(X_i).&  (t_j>\tau_m) 
  \end{align*}
  It has been shown that if $X_{t_j}$ belongs to the target area 
  then  $f(x) < \max_{i=1\dots t_j}f(X_i)$ for all $x \in H_{I^{\star}}$,
  which combined with the definition of 
  the set of potential maximizers $\X_{k,t_j}$ at time $t_j$ implies that 
  $
    H_{I^{\star}} \notin \X_{\tau_j}.
  $ 
  Hence, once an evaluation has been made 
  in $H_{I^{\star}}$, there will not 
  be any future evaluation point falling inside this cube.
  We thus deduce that $|I_{t_j}| \leq |I_{t_j-1}|-1$ for all $j \geq 1$ 
  which leads us to the following contradiction:
  \[
      0 \leq  |I_{t_{N^{d}+1}}| 
      = |I_{\tau_m}| + \sum_{j=\tau_m+1}^{t_{N^{d}+1}} |I_{t_j}| - |I_{t_{j-1}}| \leq
      |I_{\tau_m}| -(N^d+1) \leq N^d- (N^d+1) <0
  \]
  and proves the statement.
\end{proof}
  Based on this lemma, one might then derive a bound on 
  the stopping time $\tau_{m+1}$.
  \begin{corollary}
  \label{coro:stopping}
   For all $m \in \{1,\dots, M-1 \}$, 
   we have on the event $\bigcap_{l=1}^m E_l$ that
   \[
     \tau_{m+1} \leq N'_{1} +
     \sum_{l=1}^m \left( N'_{l+1} + N_l \right).
   \]
  \end{corollary}

  \begin{proof}
  The result is proved by induction. We start with the case where $m = 1$.
  Assuming that $E_1$ holds true and 
  observing that (i) $\tau_1 =N'_1$ and (ii)
  $\X \subseteq B(x^{\star},D) =
  B(x^{\star},D/2) \cup B(x^{\star},D)/ B(x^{\star},D/2)$,
  one can then write:
  \begin{align*}
   \tau_2 & = \tau_1 + \sum_{i=\tau_1 +1}^{\tau_2} \indic{X_i \in B(x^{\star}, D)  
    } \\
    & = N'_1 + \sum_{i=\tau_1 +1}^{\tau_2} \indic{X_i \in B(x^{\star}, D/2) }  
    + \sum_{i=\tau_1 +1}^{\tau_2} \indic{X_i \in B(x^{\star}, D) /
    B(x^{\star}, D/2)}.
  \end{align*}
  However, since (i)
  $
   \sum_{i=\tau_1 +1}^{\tau_2} \indic{X_i \in B(x^{\star}, D/2) } =N'_2
  $
  by definition of $\tau_2$ and (ii)
  $
   \sum_{i=\tau_1 +1}^{\tau_2} \indic{X_i \in B(x^{\star}, D) /
    B(x^{\star}, D/2)} \leq N_1
  $
  by Lemma \ref{lem:covering}, the result holds true for $m=1$.
  Consider now any $m\geq2$ and assume  that the statement holds true for all $l < m$.
  Again, observing that 
  $\X \subseteq 
  B(x^{\star}, D\cdot 2^{-m}) \cup 
  \bigcup_{l=1}^{m} B(x^{\star}, D\cdot 2^{-(l-1)} ) /B(x^{\star}, D\cdot 2^{-l} )$
  and keeping in mind that the stopping times are bounded by the
  induction assumption,
  one can write
    \begin{flushleft}
   $\displaystyle
   \tau_{m+1} = \tau_m 
  + \sum_{i=\tau_m+1}^{\tau_{m+1}} \indic{X_i \in B(x^{\star},D\cdot 2 ^{-m} )} $
    \end{flushleft}
    \begin{flushright}
    $\displaystyle
    + \sum_{i=\tau_m+1}^{\tau_{m+1}} \sum_{l=1}^{m}
  \indic{X_i \in B(x^{\star},D\cdot 2^{-(l-1)} ) / B(x^{\star}, D \cdot 2^{-l})}.$
  \end{flushright}
  Now, combining the telescopic representation
  $\tau_{m+1} = \tau_1 + \sum_{l=1}^m (\tau_{l+1} -\tau_l) $
  with the previous decomposition  gives that
  \begin{flushleft}
   $\displaystyle
   \tau_{m+1} = \tau_1 
   + \sum_{l=1}^m \sum_{\tau_l+1}^{\tau_{l+1}} \indic{X_i \in B(x^{\star},D \cdot 2^{-l} )}$
    \end{flushleft}
    \begin{flushright}
    $\displaystyle
    + \sum_{l=1}^{m} \sum^{\tau_{m+1}}_{i= \tau_l+1} 
    \indic{X_i \in B(x^{\star}, D\cdot 2^{-(l-1)})/B(x^{\star}, D\cdot  2^{-l})}.$
  \end{flushright}
  However, since (i)
  $
  \tau_1 =N'_1 \text{,~(ii)~}
  \sum_{\tau_l+1}^{\tau_{l+1}} \indic{X_i \in B(x^{\star},D \cdot 2^{-l} )}=N'_{l+1},$
  for all $l\geq 1$
  by definition of the stopping times 
  and (iii)
  $
  \sum^{\tau_{m+1}}_{i= \tau_l+1} 
    \indic{X_i \in B(x^{\star}, 2^{-(l-1)})/B(x^{\star}, 2^{-l})} \leq N_{l},$
    for all $l \geq 1$
  on the event  $\bigcap_{l=1}^mE_l$, from Lemma \ref{lem:covering}, we finally get that
  \[
  \tau_{m+1} \leq N'_1 + \sum_{l=1}^{m} (N'_{l+1}  + N_{l}).
  \]
\end{proof}
  
  \noindent As Corollary \ref{coro:stopping} gives the desired bound on $\tau_{m+1}$,
  it remains to control the maximum $\max_{i=1 \dots \tau_{m+1}}f(X_i)$.
  The next lemma shows that i.i.d.~results can actually be used to bound this term.\\
  \begin{lemma}
  \label{lem:bonus2}
   For all $m \in \{1, \dots, M-1 \}$, 
   we have that $\forall y \in\text{\normalfont{Im}}(f)$,
   \[
    \P\left( \max_{i = 1\dots \tau_{m+1}}f(X_i) \geq y 
    ~|~ \bigcap_{l=1}^m  E_l\right) 
    \geq 
      \P\left( \max_{i = 1\dots N'_{m+1}}f(X'_i) \geq y \right).
   \]
   where $X'_1\dots X'_{N'_{m+1}}$ denotes a sequence 
   $N'_{m+1}$ 
   i.i.d.~copies of $X' \sim \mathcal{U}(\X\cap B(x^{\star}, D\cdot 2^{-m}))$.
  \end{lemma}

  \begin{proof} 
  From Corollary \ref{coro:stopping}, we know that
  on the event $\bigcap_{l=1}^mE_l$
  the stopping time $\tau_{m+1}$ is finite.
  Moreover, as $\sum_{i=\tau_m+1}^{\tau_{m+1}} \indic{ X_i \in B(x^{\star}, 
  D \cdot 2^{-m})}= N_{m+1}'$
  by definition of $\tau_{m+1}$,
  it can then easily be shown by reproducing
  the same steps as in the proof of Proposition \ref{prop:fasterprs} 
  with the evaluations points falling into $B(x^{\star}, D\cdot 2^{-m})$ after $\tau_m$
  that the algorithm is faster than a Pure Random Search performed
  over the subspace $\X \cap B(x^{\star}, D\cdot 2^{-m})$, which 
  proves the result.
  \end{proof}
  As a direct consequence of this lemma, 
  one can get the desired bound on the maxima as shown in the next corollary.\\
  \begin{corollary}
  \label{eq:coro}
   For all $m \in \{1, \dots, M-1 \}$, we have that
   \[
    \P\left( \max_{i = 1\dots \tau_{m+1}}f(X_i) \geq \max_{x \in \X} f(x)
    - \frac{c_k}{2} \cdot \left( \frac{D}{2^{m+1}}  \right)^{\kappa}
    ~|~ \bigcap_{l=1}^m E_l \right) \geq 1 - \delta/M.
   \]
  \end{corollary}

  \begin{proof}
   Omitting the conditionning upon $\bigcap_{l=1}^m E_l$,
   we obtain from the combination of Lemma \ref{lem:bonus2} and
   Proposition \ref{prop:cvg_prs}
   that with probability at least $1-\delta/M$:
  \begin{align*}
   \max_{x \in \X}f(x) - \max_{i=1\dots \tau_{m+1}} f(X_i)
    & \leq k \cdot 2D \cdot 2^{-m} \cdot
   \left(  \frac{\ln(M/\delta)}{N'_{m+1}} \right)^{\frac{1}{d}}  \\
    & \leq 
    k \cdot 2D \cdot 2^{-m} \cdot 
    \left(  \frac{\ln(M/\delta)}{\ln(M/\delta)2^{d(m+1)(\kappa-1)}}
    \left(  \frac{c_{\kappa}D^{\kappa}}{8kD} \right)^d
    \right)^{\frac{1}{d}} \\
    & = \frac{c_{\kappa}}{2} \cdot \left(  \frac{D}{2^{m+1}} \right)^{\kappa}.
  \end{align*}
  \end{proof}
  At this point, 
  we know from the combination of Corollary \ref{coro:stopping} and Corollary \ref{eq:coro}
  that
  $$
   \forall m \in\{ 1,\dots, M-1\},~~\P\left(E_{m+1}|\bigcap_{l=1}^{m} E_l
   \right)\geq 1-\delta/M
  $$
  which proves from (\ref{eq:final}) that $\P(E_M) \geq 1-\delta$.

  ~\\
  {\bf Analysis of $E_M$.}
  As 
  $
       \max_{i=1 \dots \tau_M} f(X_i)
   \geq \max_{x \in \X}f(x)  - \frac{c_{\kappa}}{2} 
   \cdot D^{\kappa}\cdot 2^{-M\kappa}
   $
   and 
      $
      \tau_M \leq  N'_1  
      + \sum_{l=1}^{M-1} \left( N'_{l+1}  + N_l \right)
      $
    on the event $E_M$,
  it remains to show that 
  $
   N'_1 + \sum_{l=1}^{M-1}  \left(N'_{l+1} + N_l \right) \leq n
  $
  to conclude the proof.
  Consider first the case $\kappa = 1$. 
  Setting $C= (8k/c_{\kappa})^d$ and observing that (i)
  $
   N'_l \leq  \ln(M/\delta) C +1,
   $
   (ii)
   $
   N_l \leq 2 \cdot C \cdot (2\sqrt{d})^d -1
   $
   for all $l \leq M$
   and (iii) 
   $ M\leq n$, one gets:
  \begin{align*}
   N'_1 + \sum_{l=1}^{M-1}  \left(N'_{l+1} + N_l \right) 
   & \leq C \cdot M  \left( \ln(M/\delta) + 2(2\sqrt{d})^d \right) \\
    & \leq n \cdot \frac{\ln(M/\delta) +2(2\sqrt{d})^d}{\ln(n/\delta)+ 2(2\sqrt{d})^d} \\
    & \leq n.
  \end{align*}
  For $\kappa>1$, 
  since (i) $M$ was chosen so that
  $
   \frac{2^{d(\kappa-1)M}-1}{2^{d(\kappa-1)}-1} 
   \leq \frac{n}{C}\cdot \frac{1}{\ln(n/\delta)+ 2(2\sqrt{d})^d}
  $
  and (ii) $M\leq n$,
  we obtain:
    \begin{align*}
   N'_1 + \sum_{l=1}^{M-1}  \left(N'_{l+1} + N_l \right) 
   & \leq C\cdot \left( \ln(M/\delta)+2(2\sqrt{d})^d  \right) 
   \sum_{l=1}^M (2^{d(\kappa-1)})^l \\
    & \leq 
   C\cdot \left( \ln(M/\delta)+2(2\sqrt{d})^d  \right)
   \cdot  \frac{2^{d(\kappa-1)M}-1}{2^{d(\kappa-1)}-1} \\
   & \leq n.
  \end{align*}
  Finally, using the elementary inequality $\floor{x}\geq x-1$ 
  over $M$
  and 
  the inequality $c_{\kappa}D^{\kappa} \leq k \diam{\X}$ (by Condition \ref{cond})
  leads  to the desired result and completes the proof.\hfill\(\Box\)

 ~\\
  \noindent {\bf Proof of Theorem \ref{th:lower_bound}.} {\sc(Lower bound)} 
  Pick any $n\in \mathbb{N}^{\star}$ and  $\delta \in (0,1)$, set
  $\epsilon = c_{\kappa} \rad{\X}^{\kappa} \delta^{\kappa/d} \exp(-\kappa(n-\sqrt{2n\ln(1/\delta)})/d)$,
  let $\X_{\epsilon} = \{x \in \X: f(x) \geq \max_{x\in\X}f(x) - \epsilon \}$ 
  be the corresponding level set.
  Observe first that since (i)
  $\X_{\epsilon}
  =\{x \in \X: \epsilon \geq f(x^{\star})-f(x)  \} 
  \subseteq \{x \in \X: \epsilon \geq c_{\kappa} \norm{x^{\star} - x }_2^{\kappa} \}
   = \X \cap B(x^{\star}, (\epsilon/c_{\kappa})^{1/\kappa}
  $
  and (ii) there exists $x \in \X$ such that $B(x,\rad{\X}) \subseteq \X$, then
  $\mu(\X_{\epsilon})/\mu(\X) 
  \leq ( (\epsilon/c_{\kappa})^{1/\kappa} /\rad{\X} )^d =
  \delta e^{-n-\sqrt{2n\ln(1/\delta)}}$.
  It can then easily be shown by reproducing the same steps as in the proof
  of the Lower bound of Theorem 17
  in (\cite{malherbe2016ranking}) that
  \begin{align*}
   \P\left( \max_{i=1\dots n}f(X_i) \geq \max_{x \in \X} f(x) - \epsilon \right) 
   & = \P\left( \frac{ \mu(\{x \in \X: f(x)
   \geq \underset{i=1\dots n}{\max}f(X_i) \}) }{\mu(\X)}
   \leq  \frac{\mu(\X_{\epsilon})}{ \mu(\X) } \right)\\
    &\leq 
   \P\left(\prod_{i=1}^nU_i \leq \frac{\mu(\X_{\epsilon})}{ \mu(\X) } \right) \\
     & \leq \P \left( \prod_{i=1}^nU_i \leq 
      \delta \cdot e^{-n-\sqrt{2n\ln(1/\delta)}} \right)  \\
     & = \P\left( \sum_{i=1}^n -\ln(U_i) > n + \sqrt{2n\ln(1/\delta)} + \ln(1/\delta) \right)\\
     & \leq \delta
  \end{align*}
  where $U_1,\dots, U_n$ denotes a sequence of $n$ i.i.d.~copies of $U\sim\mathcal{U}([0,1])$.
  We point out that a concentration results for gamma random variable
  was used on the last line (see Lemma 37 and Lemma 38
  in (\cite{malherbe2016ranking}) for more details).\hfill\(\Box\)

  \section{Analysis of AdaLipOpt (proofs of Section \ref{sec:adalipopt})}
  
  {\bf Proof of Proposition \ref{prop:lip_estimate}.} 
  Pick any $t\geq 2$, 
  consider any non-constant $f \in \bigcup_{k\geq 0}\Lip(k)$  and set
  $i^{\star} = \min\{i \in \mathbb{Z}: f \in \text{Lip}(k_i) \}$. 
  To prove the result, we decorelate the sample  and use the fact that 
  $(X_1, X_{\lfloor t/2 \rfloor+1})
  , \dots, 
  (X_{\lfloor t/2 \rfloor}, X_{2\lfloor t/2 \rfloor})$
  forms a sequence of ${\lfloor t/2 \rfloor}$ i.i.d.~copies of 
  $(X,X') \sim \mathcal{U}(\X \times \X)$:
  \begin{align*}
   \P\left( f \in \text{Lip}(\hat{k}_t)  \right)
   &= \P\left( \hat{k}_t = k_{i^{\star}}  \right) \\
   &= \P\left( \bigcup_{i \neq j}^t \left\{ 
   \abs{f(X_i) -f(X_j)} > k_{i^{\star}-1} \cdot \norm{X_i - X_j}_2
   \right\} \right)\\
    &\geq \P\left( \bigcup_{i=1}^{\lfloor t/2 \rfloor}  \left\{
    \abs{f(X_i) -f(X_{\lfloor t/2 \rfloor+i})} > k_{i^{\star}-1} 
    \cdot \norm{X_i - X_{\lfloor t/2 \rfloor+i}}_2
    \right\} \right) \\
    & = 1 -\P\left( \frac{ \abs{f(X_1) -f(X_2)}  }{~\norm{X_1 - X_{2}}_2} \leq k_{i^{\star}-1}
   \right)^{\lfloor  t/2 \rfloor} \\
   & = 1- \left(1 - \Gamma(f,k_{i^{\star}-1}) \right)^{\lfloor  t/2 \rfloor}.
  \end{align*}
  It remains to show that 
  $\Gamma(f,k_{i^{\star}-1})>0$.
  Observe first that since $f\in \Lip(k_{i^{\star}})$, then the function
  $
   F: (x,x') \mapsto \abs{f(x) - f(x')} - k_{i^{\star}-1} \cdot \norm{x-x'}_2
  $
  is also continuous. 
  However, as $f \notin$Lip$(k_{i^{\star}-1})$, we know that there
  exists some $(x_1,x_2 ) \in \X\times \X$ such that $F(x_1,x_2) > 0$.
  Hence, it follows from the continuity of $F$
  that there necessarily exists some $\epsilon>0$ such that 
  $\forall (x,x') \in B(x_1,\epsilon)\cap\X \times B(x_2, \epsilon)\cap\X$, $F(x,x') > 0$
  which proves the proof.\hfill\(\Box\)
  
  ~\\
  {\bf Proof of Proposition \ref{prop:consistency_adalipopt}.} 
  Combining the consistency equivalence of Proposition 
  \ref{prop:consistency_equivalence} with the upper bound 
  on the covering rate obtained in Example \ref{ex:PRS} gives the result.\hfill\(\Box\)

  ~\\
  {\bf Proof of Proposition \ref{prop:adagenericrate}}. Fix any $\delta \in (0,1)$,
  set $N_1 = 2+\ceil{2\ln(\delta/3)/\ln(1-\Gamma(f,k_{i^{\star}\minus 1}))}$ and 
  $
  N_2 = \ceil{((\sqrt{\ln(3/\delta)/2 + 4 N_1p} -\sqrt{ \ln(3/\delta)/2 }) /2p)^2}.
  $
  Considering any $n> N_2$, we prove the result in three steps.
  
  \noindent{\bf Step 1.} As the constant $N_1$ and $N_2$ were chosen so that
  Hoeffding's inequality ensures that
   $ \P\left(   \sum_{i=1}^{N_2} B_i \geq N_1 \right) \geq 1- \delta/3$,
   we know that after $N_2$ iterations and with probability $1-\delta/3$
   we have collected at 
   least $N_1$ evaluation points randomly and uniformly distributed over $\X$ 
   due to the  exploration step.\\

   \noindent{\bf Step 2.} 
   Using Proposition \ref{prop:lip_estimate} and 
   the fist $N_1$ evaluation points which have been 
   sampled independently and uniformly over $\X$, we know 
   that after $N_2$ iterations and on the event 
   $\{\sum_{i=1}^{N_2} B_i \geq N_1\}$
   the Lipschtz constant $k_{i^{\star}}$ has been estimated with probability at 
   least $1-\delta/3$, {\it i.e.},
   $
    \P\left( \forall t \geq N_2+1,~  \hat{k}_t = k_{i^{\star}}
    \mid \sum_{i=1}^{N_2} B_i \geq N_1  \right) \geq 1- \delta/3.
   $
   
  ~\\
   {\bf Step 3.} 
   Finally, as the Lipschtz constant estimate $\hat{k}_t$ satisfies $f\in\Lip(\hat{k}_t)$
   for all $t\geq N_2+1$ on the above event,
   one can easily show by reproducing the same steps as in Proposition \ref{prop:fasterprs}
   that
   contionned upon the event $\{\forall t \geq N_2+1,~  \hat{k}_t = k_{i^{\star}}\}
    \cap\{ \sum_{i=1}^{N_2} B_i \geq N_1 \}$
   the algorithm is always faster or equal to a Pure Random Search ran
   with $n-N_2$ i.i.d.~copies of $X' \sim\mathcal{U}(\X)$.
  Therefore, using (i) the bound of Proposition \ref{prop:cvg_prs},
  (ii) the elementaries inequalities 
  $\ceil{x} \leq x +1$, $\floor{x}\geq x-1$,
  $\sqrt{x+y} - \sqrt{x} \leq  \sqrt{y}$
  and (iv) the definition of $N_2<n$, we obtain that with probability at least 
  $(1-\delta/3)^3 \geq 1-\delta$,
  \begin{align*}
   \max_{x\in\X}f(x) - \max_{i=1\dots n}f(X_i) 
   & \leq k_{i^{\star}} \cdot
   \diam{\X} \cdot \left(  \frac{\ln(3/\delta)}{ n-N_2 } \right)^{\frac{1}{d}} \\
   & = k_{i^{\star}} \cdot
   \diam{\X} \cdot \left(  \frac{n}{n-N_2} \right)^{\frac{1}{d}} \cdot
   \left(  \frac{\ln(3/\delta)}{ n } \right)^{\frac{1}{d}}  \\
   & \leq k_{i^{\star}} \cdot
   \diam{\X} \cdot \left(  1+N_2 \right)^{\frac{1}{d}}
   \left(  \frac{\ln(3/\delta)}{ n } \right)^{\frac{1}{d}}\\
   & \leq k_{i^{\star}} \cdot \diam{\X} 
   \cdot \left( \frac{5}{p} + \frac{2\ln(\delta/3)}{p\ln(1-\Gamma(f, k_{i^{\star}\minus1}))}  \right)^{\frac{1}{d}} 
   \cdot \left( \frac{\ln(3/\delta)}{n} \right)^{\frac{1}{d}}.
  \end{align*}
  The result is extended to the case where $n\leq N_2$ by noticing that 
  the bound is superior to $k_{i^{\star}}\cdot \diam{\X}$ 
  in that case, and thus trivial.\hfill\(\Box\)

    ~\\
  {\bf Proof of Proposition \ref{th:fastada}}. Fix $\delta \in (0,1)$, 
  set $N_1 = 2+\ceil{2\ln(4/\delta)/\ln(1-\Gamma)}$ and
  $
  N_2 = \ceil{((\sqrt{\ln(4/\delta)/2 + 4 N_1p} -\sqrt{ \ln(4/\delta)/2 }) /2p)^2}
  $
  and let
  $
  N_3 = N_2 + \ceil{   2\ln(4/\delta) /(1-p)^2 }.
  $
  Picking any $n> N_3$, we proceed
  similarly as in the proof of Proposition \ref{prop:adagenericrate}
  in four steps:\\

  
  \noindent {\bf Steps 1 \& 2.}
  As in the above prove, by definition of $N_1$ and $N_2$ and due to Hoeffding's inequality
  and Proposition \ref{prop:lip_estimate}, 
  we know that the following event: 
  $ 
  \{\forall t \geq N_2+1, \hat{k}_t = k_{i^{\star}} \} 
  \cap \{ \sum_{i=1}^{N_2} B_i \geq N_1 \}  $
  holds true with probability at least $(1-\delta/4)^2$.

  
  ~\\
  {\bf Step 3.} Again, using Hoeffding's inequality  and the definition of
  $N_2$ and $N_3$,
  we know that after the iteration $N_2+1$ we have collected
  with probability at least $1-\delta/4$ at least 
  $(1-p)(n-N_3)/2$ exploitative evaluation points:
  \begin{align*}
   \sum_{i=N_2+1}^n \indic{B_i = 0} 
    \geq (1-p)(n-N_2) - \sqrt{  \frac{(n-N_2)\ln(4/\delta)}{2} }
  \geq \frac{1-p}{2} \cdot (n-N_3).
  \end{align*}

  ~\\
  {\bf Step 4.} Reproducing the same steps as in the proof of the fast rate of Theorem \ref{th:fast_rates}
  with the $(1-p) \cdot (n-N_3)/2$
  previous exploitative points and putting the previous results altogether
  gives that with proability at least $(1-\delta/4)^4\geq 1-\delta$,
  \begin{flushleft}
   $
   \displaystyle
   \max_{x \in \X}f(x) -\max_{i=1\dots n} f(X_i) \leq k_{i^{\star}} \times \diam{\X} \times
   $
  \end{flushleft}  
  \begin{flushright}
  $
  \begin{dcases*}
       \exp\left\{-~  C_{k, \kappa} \cdot 
	\frac{(1-p)(n-N_3)\ln(2)}{2\ln(4n/\delta) + 4(2\sqrt{d})^d }
	\right\},   \text{~~~~~~~~~~~~~~~~~~~~~~}\  \kappa =1  \\
	~\\
      \frac{2^{\kappa}}{2}
	\left( 1 + C_{k, \kappa}
     \cdot  \frac{(1-p)(n-N_3)(2^{d(\kappa~\!\minus~\!1)}-1)}{2\ln(4n/\delta) +4(2\sqrt{d})^d}
	\right)^{-\frac{\kappa}{d(\kappa-1)}}\!\!\!\!\!\!\!\!\!\!\!\!\!
	\!\!\!\!\!\!\text{~,~}
	~~~~~~~~~~~\kappa>1.
    \end{dcases*}
    $
  \end{flushright}
  We now take out the term $N_3$.
  Since $C_{k_{i^{\star}},\kappa}(1-p)\leq 1$, 
  when $\kappa=1$, we have
  \[
   \max_{x \in \X}f(x) - \max_{i=1\dots n}f(X_i) \leq k_{i^{\star}} \cdot \diam{\X} \cdot   
   \exp(5N_3/2) 
   \exp\left\{-~  C_{k, \kappa} \cdot 
	\frac{(1-p)n\ln(2)}{2\ln(4n/\delta) + 4(2\sqrt{d})^d }
	\right\}.
  \]
For $\kappa>1$, setting $C= C_{k_{i^{\star},\kappa}}(1-p)/(2\ln(4n/\delta) +4(2\sqrt{d})^d)$
and using the decomposition $n = (n- N_3)+N_3$,
we bound the ratio:
\begin{align*}
 \left(   
 \frac{1+C n (2^{d(\kappa-1}-1)}{1+C (n-N_3) (2^{d(\kappa-1)}-1)}
 \right)^{\frac{\kappa}{d(\kappa-1)}} 
 &\leq \left( 1+  \frac{ CN_3(2^{d(\kappa-1)}-1)}{1+C(2^{d(\kappa-1)}-1)}  \right)^{\frac{\kappa}{d(\kappa-1)}}.
\end{align*}
In the case where $\kappa/d(\kappa-1)\leq 1$, one directly obtains
\[
 \left( 1+  
 \frac{ CN_3(2^{d(\kappa-1)}-1)}{1+C(2^{d(\kappa-1)}-1)}  \right)^{\frac{\kappa}{d(\kappa-1)}}
 \leq \left(1+N_3 \right)^{\frac{\kappa}{d(\kappa -1)}} \leq (1+N_3) \leq e^{N_3}
\]
Considering the case where $\kappa/d(\kappa-1)> 1$ and setting
$\kappa =1+\epsilon/d$ with $\epsilon \in (0,1)$,
we obtain from the inequalities (i) $\kappa\leq 1+1/d \leq 2$, 
(ii) $\forall \epsilon\in(0,1)$, $2^{\epsilon}-1 \leq \epsilon$ and
(iii) $C\leq 1/2$ that
\[
  \left(1+\frac{ CN_3(2^{d(\kappa-1)}-1)}{1+C(2^{d(\kappa-1)}-1)}  \right)^{\frac{\kappa}{d(\kappa-1)}}
  \leq \left(1 + CN_3(2^{\epsilon}-1) \right)^{\frac{2}{\epsilon}}
  \leq (1+CN_3 \epsilon)^{\frac{2}{\epsilon}} \leq e^{2CN_3} \leq e^{N_3}
\]
Finally, using standard bounds on $N_3$ and noticing that the previous
bound is superior to $k_{i^\star} \diam{\X}$ whenever $n \leq N_3$, 
the previous result remains valid for any $n \in \mathbb{N}^{\star}$.\hfill\(\Box\)

%

%
%
%


\begin{thebibliography}{32}
\providecommand{\natexlab}[1]{#1}
\providecommand{\url}[1]{\texttt{#1}}
\expandafter\ifx\csname urlstyle\endcsname\relax
  \providecommand{\doi}[1]{doi: #1}\else
  \providecommand{\doi}{doi: \begingroup \urlstyle{rm}\Url}\fi

\bibitem[Bubeck et~al.(2011)Bubeck, Stoltz, and Yu]{bubeck2011lipschitz}
S{\'e}bastien Bubeck, Gilles Stoltz, and Jia~Yuan Yu.
\newblock Lipschitz bandits without the lipschitz constant.
\newblock In \emph{International Conference on Algorithmic Learning Theory},
  pages 144--158. Springer, 2011.

\bibitem[Bull(2011)]{bull2011convergence}
Adam~D. Bull.
\newblock Convergence rates of efficient global optimization algorithms.
\newblock \emph{The Journal of Machine Learning Research}, 12:\penalty0
  2879--2904, 2011.

\bibitem[Dasgupta(2011)]{dasgupta2011two}
Sanjoy Dasgupta.
\newblock Two faces of active learning.
\newblock \emph{Theoretical Computer Science}, 412\penalty0 (19):\penalty0
  1767--1781, 2011.

\bibitem[Finkel and Kelley(2004)]{finkel2004convergence}
Daniel~E Finkel and CT~Kelley.
\newblock Convergence analysis of the direct algorithm.
\newblock \emph{Optimization On-line Digest}, 2004.

\bibitem[Grill et~al.(2015)Grill, Valko, and Munos]{grill2015black}
Jean-Bastien Grill, Michal Valko, and R{\'e}mi Munos.
\newblock Black-box optimization of noisy functions with unknown smoothness.
\newblock In \emph{Neural Information Processing Systems}, 2015.

\bibitem[Hanneke(2011)]{hanneke2011rates}
Steve Hanneke.
\newblock Rates of convergence in active learning.
\newblock \emph{The Annals of Statistics}, 39\penalty0 (1):\penalty0 333--361,
  2011.

\bibitem[Hansen(2006)]{hansen2006cma}
Nikolaus Hansen.
\newblock The cma evolution strategy: a comparing review.
\newblock In \emph{Towards a New Evolutionary Computation}, pages 75--102.
  Springer, 2006.

\bibitem[Hansen(2011)]{CMAES_implementation}
Nikolaus Hansen.
\newblock The cma evolution strategy: A tutorial.
\newblock Retrieved May 15, 2016, from
  \url{http://www.lri.fr/hansen/cmaesintro.html}, 2011.

\bibitem[Huyer and Neumaier(1999)]{huyer1999global}
Waltraud Huyer and Arnold Neumaier.
\newblock Global optimization by multilevel coordinate search.
\newblock \emph{Journal of Global Optimization}, 14\penalty0 (4):\penalty0
  331--355, 1999.

\bibitem[Jamil and Yang(2013)]{jamil2013literature}
Momin Jamil and Xin-She Yang.
\newblock A literature survey of benchmark functions for global optimization
  problems.
\newblock \emph{International Journal of Mathematical Modelling and Numerical
  Optimisation}, 4\penalty0 (2):\penalty0 150--194, 2013.

\bibitem[Johnson(2014)]{johnson2014nlopt}
Steven~G. Johnson.
\newblock The {NL}opt nonlinear-optimization package.
\newblock Retrieved May 15, 2016, from \url{http://ab-initio.mit.edu/nlopt},
  2014.

\bibitem[Jones et~al.(1993)Jones, Perttunen, and
  Stuckman]{jones1993lipschitzian}
Donald~R Jones, Cary~D Perttunen, and Bruce~E Stuckman.
\newblock Lipschitzian optimization without the lipschitz constant.
\newblock \emph{Journal of Optimization Theory and Applications}, 79\penalty0
  (1):\penalty0 157--181, 1993.

\bibitem[Jones et~al.(1998)Jones, Schonlau, and Welch]{jones1998efficient}
Donald~R. Jones, Matthias Schonlau, and William~J. Welch.
\newblock Efficient global optimization of expensive black-box functions.
\newblock \emph{Journal of Global Optimization}, 13\penalty0 (4):\penalty0
  455--492, 1998.

\bibitem[Kaelo and Ali(2006)]{kaelo2006some}
Professor Kaelo and Montaz Ali.
\newblock Some variants of the controlled random search algorithm for global
  optimization.
\newblock \emph{Journal of Optimization Theory and Applications}, 130\penalty0
  (2):\penalty0 253--264, 2006.

\bibitem[Kan and Timmer(1987)]{kan1987stochastic}
AHG~Rinnooy Kan and Gerrit~T. Timmer.
\newblock Stochastic global optimization methods part i: Clustering methods.
\newblock \emph{Mathematical Programming}, 39\penalty0 (1):\penalty0 27--56,
  1987.

\bibitem[Lichman(2013)]{Lichman:2013}
Moshe Lichman.
\newblock {UCI} machine learning repository, 2013.
\newblock URL \url{http://archive.ics.uci.edu/ml}.

\bibitem[Malherbe and Vayatis(2016)]{malherbe2016ranking}
C{\'e}dric Malherbe and Nicolas Vayatis.
\newblock A ranking approach to global optimization.
\newblock \emph{arXiv preprint arXiv:1603.04381}, 2016.

\bibitem[Malherbe et~al.(2016)Malherbe, Contal, and Vayatis]{malherbe}
C{\'e}dric Malherbe, Emile Contal, and Nicolas Vayatis.
\newblock A ranking approach to global optimization.
\newblock In \emph{In Proceedings of the 33st International Conference on
  Machine Learning}, pages 1539--1547, 2016.

\bibitem[Martinez-Cantin(2014)]{martinez2014bayesopt}
Ruben Martinez-Cantin.
\newblock Bayesopt: A bayesian optimization library for nonlinear optimization,
  experimental design and bandits.
\newblock \emph{The Journal of Machine Learning Research}, 15\penalty0
  (1):\penalty0 3735--3739, 2014.

\bibitem[Mladineo(1986)]{mladineo1986algorithm}
Regina~Hunter Mladineo.
\newblock An algorithm for finding the global maximum of a multimodal,
  multivariate function.
\newblock \emph{Mathematical Programming}, 34\penalty0 (2):\penalty0 188--200,
  1986.

\bibitem[Munos(2014)]{munosmono}
R{\'e}mi Munos.
\newblock From bandits to monte-carlo tree search: The optimistic principle
  applied to optimization and planning.
\newblock \emph{Foundations and Trends{\textregistered} in Machine Learning},
  7\penalty0 (1):\penalty0 1--129, 2014.

\bibitem[Pint{\'e}r(1991)]{pinter1991global}
J{\'a}nos~D. Pint{\'e}r.
\newblock Global optimization in action.
\newblock \emph{Scientific American}, 264:\penalty0 54--63, 1991.

\bibitem[Piyavskii(1972)]{piyavskii1972algorithm}
SA~Piyavskii.
\newblock An algorithm for finding the absolute extremum of a function.
\newblock \emph{USSR Computational Mathematics and Mathematical Physics},
  12\penalty0 (4):\penalty0 57--67, 1972.

\bibitem[Preux et~al.(2014)Preux, Munos, and Valko]{preux2014bandits}
Philippe Preux, R{\'e}mi Munos, and Michal Valko.
\newblock Bandits attack function optimization.
\newblock In \emph{Evolutionary Computation (CEC), 2014 IEEE Congress on},
  pages 2245--2252. IEEE, 2014.

\bibitem[Rios and Sahinidis(2013)]{rios2013derivative}
Luis~Miguel Rios and Nikolaos~V. Sahinidis.
\newblock Derivative-free optimization: a review of algorithms and comparison
  of software implementations.
\newblock \emph{Journal of Global Optimization}, 56\penalty0 (3):\penalty0
  1247--1293, 2013.

\bibitem[Shubert(1972)]{shubert1972sequential}
Bruno~O Shubert.
\newblock A sequential method seeking the global maximum of a function.
\newblock \emph{SIAM Journal on Numerical Analysis}, 9\penalty0 (3):\penalty0
  379--388, 1972.

\bibitem[Stein(1987)]{stein1987large}
Michael Stein.
\newblock Large sample properties of simulations using latin hypercube
  sampling.
\newblock \emph{Technometrics}, 29\penalty0 (2):\penalty0 143--151, 1987.

\bibitem[Surjanovic and Bingham(2013)]{simulationlib}
Sonja Surjanovic and Derek Bingham.
\newblock Virtual library of simulation experiments: Test functions and
  datasets.
\newblock Retrieved May 15, 2016, from \url{http://www.sfu.ca/~ssurjano}, 2013.

\bibitem[Valko et~al.(2013)Valko, Carpentier, and Munos]{valko2013stochastic}
Michal Valko, Alexandra Carpentier, and R{\'e}mi Munos.
\newblock Stochastic simultaneous optimistic optimization.
\newblock In \emph{In Proceedings of the 30th International Conference on
  Machine Learning}, pages 19--27, 2013.

\bibitem[Wood and Zhang(1996)]{wood1996estimation}
GR~Wood and BP~Zhang.
\newblock Estimation of the lipschitz constant of a function.
\newblock \emph{Journal of Global Optimization}, 8\penalty0 (1):\penalty0
  91--103, 1996.

\bibitem[Zabinsky and Smith(1992)]{zabinsky1992pure}
Zelda~B. Zabinsky and Robert~L Smith.
\newblock Pure adaptive search in global optimization.
\newblock \emph{Mathematical Programming}, 53\penalty0 (1-3):\penalty0
  323--338, 1992.

\bibitem[Zhigljavsky and Pint{\'e}r(1991)]{zhigljavsky1991theory}
A.A. Zhigljavsky and J.D. Pint{\'e}r.
\newblock \emph{Theory of Global Random Search}.
\newblock Mathematics and its Applications. Springer Netherlands, 1991.
\newblock ISBN 9780792311225.

\end{thebibliography}

\end{document}